\documentclass[10pt,journal]{IEEEtran}
\usepackage{caption}
\usepackage{subfigure}

\usepackage{bbding}
\usepackage{epsfig}
\usepackage{float}
\usepackage{graphicx}
\usepackage{amsmath,amsthm}               
{
\theoremstyle{plain}
\newtheorem{assumption}{Assumption}
}
\usepackage{amssymb}
\usepackage{amsthm}
\usepackage[linesnumbered,ruled]{algorithm2e}
\usepackage{colortbl,booktabs}
\usepackage{threeparttable}
\usepackage{subeqnarray}
\usepackage{cases}
\usepackage{multirow}
\usepackage{array}
\usepackage{bm}
\usepackage[switch]{lineno}
\usepackage{threeparttable}
\usepackage{fixltx2e}
\usepackage{enumerate}
\usepackage{epstopdf}
\usepackage[hyphenbreaks]{breakurl}
\usepackage[hyphens]{url}

\newcolumntype{I}{!{\vrule width 2pt}}
\newlength\savedwidth

\newlength\savewidth

\newtheorem{proposition}{Proposition}

\newtheorem{definition}{Definition}

\theoremstyle{definition}

\ifCLASSOPTIONcompsoc
\else
\fi
\ifCLASSINFOpdf
\else
\fi
\begin{document}
\title{Domain-Invariant Classifier Learning for Domain Generalization via Posterior Distribution Alignment}
\title{Domain-Invariant Classifier Learning via Constrained Maximum Cross-Domain Likelihood}
\title{Constrained Maximum Cross-Domain Likelihood for Domain Generalization}
%
%
%
%

\author{Jianxin Lin,~
	Yongqiang Tang,~
	Junping Wang
	and Wensheng Zhang~
	
	\IEEEcompsocitemizethanks{
		\IEEEcompsocthanksitem J. Lin, J. Wang and W. Zhang are with the Research Center of Precision Sensing and Control, Institute of Automation, Chinese Academy of Sciences, Beijing, 100190, China,  and also with the School of Artificial Intelligence, University of Chinese Academy of Sciences, Beijing, 100049, China (e-mail: linjianxin2020@ia.ac.cn; junping.wang@ia.ac.cn; zhangwenshengia@hotmail.com).
		\IEEEcompsocthanksitem Y. Tang is with the Research Center of Precision Sensing and Control, Institute of Automation, Chinese Academy of Sciences, Beijing, 100190, China
		(e-mail: yongqiang.tang@ia.ac.cn).
	}
\thanks{}}

%
%

\markboth{ }%
{Shell \MakeLowercase{\textit{et al.}}: Bare Demo of IEEEtran.cls for Computer Society Journals}
%


\IEEEtitleabstractindextext{%
\begin{abstract}
	As a recent noticeable topic, domain generalization aims to learn a generalizable model on multiple source domains, which is expected to perform well on unseen test domains. Great efforts have been made to learn domain-invariant features by aligning distributions across domains. However, existing works are often designed based on some relaxed conditions which are generally hard to satisfy and fail to realize the desired joint distribution alignment. In this paper, we propose a novel domain generalization method, which originates from an intuitive idea that a domain-invariant classifier can be learned by minimizing the KL-divergence between posterior distributions from different domains. To enhance the generalizability of the learned classifier, we formalize the optimization objective as an expectation computed on the ground-truth marginal distribution. Nevertheless, it also presents two obvious deficiencies, one of which is the side-effect of entropy increase in KL-divergence and the other is the unavailability of ground-truth marginal distributions.
	For the former, we introduce a term named maximum in-domain likelihood to maintain the discrimination of the learned domain-invariant representation space. For the latter, we approximate the ground-truth marginal distribution with source domains under a reasonable convex hull assumption.	Finally, a Constrained Maximum Cross-domain Likelihood (CMCL) optimization problem is deduced, by solving which the joint distributions are naturally aligned. An alternating optimization strategy is carefully designed to approximately solve this optimization problem. Extensive experiments on four standard benchmark datasets, {\it \textbf{i.e.}}, Digits-DG, PACS, Office-Home and miniDomainNet, highlight the superior performance of our method.
\end{abstract}


\begin{IEEEkeywords}
	Domain generalization, domain adaptation, distribution shift, domain-invariant representation, joint distribution alignment.
\end{IEEEkeywords}}

\maketitle

\IEEEdisplaynontitleabstractindextext

%
\IEEEpeerreviewmaketitle


\section{Introduction} \label{section:introduction}
\IEEEPARstart{D}{eep} learning methods have achieved remarkable success in computer vision tasks under the assumption that train data and test data follow the same distribution. Unfortunately, this important assumption does not hold in real-world applications \cite{UDA-Long-2016}. The distribution shift between train data and test data, which are widespread in various vision tasks, is unpredictable and not even static, thus hindering the application of deep learning in reliability-sensitive scenarios. For example, in the field of medical image processing, image data from different hospitals follow different distributions due to discrepancies in imaging protocol, device vendors and patient populations \cite{DGMI-Li-2020}. Hence, the models trained on data from one hospital often suffer from performance degradation when tested in another hospital owing to the distribution shift.

To tackle the distribution shift problem, considerable efforts have been made in domain adaptation and domain generalization. Domain adaptation assumes that the target domain is accessible and attempt to align the distributions between the source domain and the target domain. However, in the setting of domain adaptation, the model inevitably needs to be retrained when the distribution of the target domain changes, which can be time-consuming and cumbersome \cite{dsa-tmi-2020}. More importantly, in many cases, there is no way to access the target domain in advance. Fortunately, domain generalization has been proposed to improve the generalization ability of models in out-of-distribution scenarios given multiple source domains, where the target domain is inaccessible  \cite{ifr-wang-2013}. 

As an active research area, many domain generalization methods have been proposed. 
Let $X$ denote an input variable, {\it i.e.}, an image, $Z=F(X)$ denote the feature extracted from $X$ by a feature extractor $F(\cdot)$ and $Y$ denote an output variable {\it i.e.}, a label. An effective and general solution to domain generalization is learning a domain-invariant representation space where the joint distribution $P(Z,Y)$ across all source domains keeps consistent \cite{ifr-wang-2013, mmdaae-cvpr-2018, Li_TDC-cian-eccv-2018, Li-TDC-cir-AAAI-2018}. Along this line, some works \cite{ifr-wang-2013, bound-2021-arxiv} try to align the marginal distribution $P(Z)$ among domains assuming that the posterior distribution $P(Y|Z)$ is stable across domains. Problematically, there is no guarantee that $P(Y|Z)$ will be invariant when aligning $P(Z)$ \cite{mbdg-nips-2021}, \cite{re-icml-2021}. 
Some methods \cite{cdann-eccv-2018} attempt to align the class-conditional distribution $P(Z|Y)$.
According to $P(Z,Y)=P(Z|Y)P(Y)$, only if the categorical distribution $P(Y)$ keeps invariant across domains, aligning the class-conditional distributions could achieve domain-invariant joint distribution \cite{Li-TDC-cir-AAAI-2018}. But this requirement is difficult to meet in practical applications.

More recently, the domain-invariant classifier, or the invariant predictor, has attracted much interest \cite{icp-jci-2018, irm-2019-arxiv, mip-2020-arxiv, ip-nips-2021, iib-AAAI-2022}. In essence, these works are performing posterior distribution alignment.
Invariant Risk Minimization (IRM) \cite{irm-2019-arxiv} seeks an invariant causal predictor, which is a simultaneously optimal classifier for all environments (domains). IRM is formalized as a hard-to-solve bi-leveled optimization problem. The invariant causal predictor realizes the conditional expectation $\mathbb{E}[Y|Z]$ alignment across domains. It is a coarse posterior distribution alignment due to the insufficiency of the conditioned expectation. Robey {\it et al} \cite{mbdg-nips-2021} propose a novel definition of invariance called G-invariance, which requires that the classifier should hold invariant prediction after $X$ is transformed to any another domain by a domain transformation model $G$. Li {\it et al} \cite{iib-AAAI-2022} propose a new formulation called Invariant Information Bottleneck (IIB), which  achieves the domain-invariant classifier by minimizing the mutual information between $Y$ and domain label given $Z$. Despite the brilliant achievements, the above methods do not take marginal distribution alignment into consideration and thus fail to realize the desired joint distribution alignment. 
In order to ensure that the joint distribution is invariant across domains, both $P(Z)$ and $P(Y|Z)$ must be considered \cite{olirda-icml-2019}. 

In this paper, we propose a novel domain generalization method that can jointly align the posterior distribution and the marginal distribution. Specifically, we formalize a general optimization objective, in which for any given sample, except for the routine empirical risk minimization, the Kullback-Leibler (KL) divergence \cite{kl-ams-1951} between posterior distributions from different domains is also minimized so that the domain-invariant classifier can be learned. To enhance the generalization ability of the learned classifier, the optimization objective is designed as an expectation computed on the ground-truth marginal distribution.
Unfortunately, the above optimization problem still has two deficiencies that must be overcome. The first issue lies in the side-effect of KL-divergence which tends to enlarge the entropy of posterior distributions. To tackle this issue, we add a new term named maximum in-domain likelihood into the overall optimization objective, such that the discrimination of the learned domain-invariant feature space is reinforced. The second issue is that the ground-truth marginal distribution is not available directly. In light of this, we propose to approximate the real-world marginal distribution with source domains under a reasonable convex hull assumption. Eventually, a concise and intuitive optimization problem namely Constrained Maximum Cross-domain Likelihood (CMCL) is deduced, by solving which we can learn a domain-invariant representation space where the joint distributions across domains are naturally aligned.

The major contributions of our paper can be summarized as follows: 

\begin{itemize}	
\item [1)] We propose a new formulation for domain generalization, which minimizes the expectation of KL-divergence between posterior distributions from different domains. We innovatively compute the expectation on the ground-truth marginal distribution, such that the generalizability of the learned model can be enhanced.

\item [2)] A constrained maximum cross-domain likelihood optimization problem is deduced by adding an objective term of maximum in-domain likelihood and a constraint of marginal distribution alignment. The former eliminates the side-effect brought by minimizing KL-divergence, and the latter makes it possible to approximate the ground-truth marginal distribution with source domains.

\item [3)] An effective alternating optimization strategy with multiple optimization stages is elaborately developed to solve the maximum cross-domain likelihood problem. Comprehensive experiments are conducted on four widely used datasets and the results  demonstrate that our CMCL achieves superior performance on unseen domains.
\end{itemize}	


\section{Related Works} \label{section:related work}
In this section, we review the related works dealing with the domain (distribution) shift problem in deep learning, which can be divided into two categories, including domain adaptation and domain generalization.
\subsection{Domain Adaptation} 
Domain adaptation aims to tackle the domain shift between a source domain and a particular target domain \cite{progressive-tnnls-2019} \cite{nem-tnnls-2020}. The goal of domain adaptation is to train models making full use of a large amount of labeled data from a source domain to perform well on the unlabeled target domain.
Most existing domain adaptation methods focus on aligning distributions between the source domain and target domain \cite{graph-tcsvt-2022}. They can be mainly divided into two categories: discrepancy measuring based methods and domain adversarial based methods. 

Discrepancy measuring based methods employ different metrics to measure the distribution disparities and then minimize them, {\it e.g.}, Maximum Mean Discrepancy (MMD) \cite{mmd}, Central Moment Discrepancy (CMD) \cite{cmd-iclr-2017}, Wasserstein distance \cite{wd-aaai-2020}. 
Deep domain confusion \cite{deepmmd-2014-arxiv} employs MMD to align marginal distributions in the deep representation space. Deep CORAL \cite{dcoral-eccv-2016} and CMD \cite{cmd-iclr-2017} align marginal distributions with moment matching. 
Joint MMD \cite{djdn-icml-2017} is proposed to align the joint distributions considering the distribution shifts may stem from joint distributions.
Domain adversarial based methods use domain discriminators to minimize the distance between distributions \cite{dann-icml-2015}. Feature extractors are optimized to confuse the discriminators so that the divergence of distributions is reduced.
Domain-adversarial neural network \cite{dann-icml-2015} is proposed to align marginal distributions by adversarial learning.
Multi-adversarial domain adaptation \cite{mada-aaai-2018} considers the alignment of multi-mode distributions, {\it i.e.}, class-conditional distributions, instead of marginal distributions.
Zuo {\it et al} \cite{ajada-tcsvt-2022} concatenate features and corresponding labels together, and feed them into a domain classifier, then the joint distributions are aligned in an adversarial training manner.

The difference between domain adaptation and domain generalization lies in the accessibility to the target domain. The former focuses on the alignment between the given source domain and target domain, but the latter focuses more on the generalizability on unseen test domains.

\subsection{Domain Generalization} 
Domain generalization aims to train models on  several source domains and test them on unseen domain \cite{wdg-tnnls-2017, emv-tnnls-2018}. Existing works of domain generalization carry out the research mainly from three aspects, including  learning strategy, data augmentation and domain invariant representation.

Learning strategy based methods mainly design special learning strategies to enhance generalizability.
Some works employ meta learning to address domain generalization, which randomly split the source domains into meta-train and meta-test to simulate the domain shift.
Balaji {\it et al} \cite{metareg-nips-2018} train a regularizer through meta learning to capture the notion of domain generalization, which is parameterized by a neural network. Dou {\it et al} \cite{masf-nips-2019} propose a model-agnostic learning paradigm based meta learning to enhance the generalizability of learned features. Global inter-class relationships, local class-specific cohesion and separation of sample features are also considered to regularize the semantic structure of the feature space.
In addition to meta learning, Distributionally Robust Optimization (DRO) \cite{dro-2016-arxiv} is also used for domain generalization, which trains models by minimizing the worst-case loss over pre-defined groups. Sagawa {\it et al} \cite{dro-arxiv-2019} find that coupling DRO with stronger regularization achieves higher worst-case accuracy in the over-parameterized regime.

The core idea of data augmentation based methods is to increase the diversity of training data. 
MixStyle \cite{mixstyle-iclr-2021} is motivated that the visual domain is closely related to image style, which is encoded by feature statistics. The domain diversity can be increased by randomly combining feature statistics between two training instances. Deep Domain-Adversarial Image Generation (DDAIG) \cite{digits-dg-AAAI-2020} is proposed to fool the domain classifier by augmenting images. A domain transformation network is designed to automatically change image style. Seo {\it et al} \cite{dson-eccv-2020} propose a Domain-Specific Optimized Normalization (DSON) to remove domain-specific style. 
Wang {\it et al} \cite{style-tcsvt-2022} design a feature-based style randomization module, which randomizes image style by introducing random noise into feature statistics.
These style augmentation based methods actually exploit the prior knowledge about domain shift, that is, the difference across source domains lies in image style. Though they work well in existing benchmarks, style augmentation based methods would probably fail when the domain shift is caused by other potential factors. Methods which do not rely on prior knowledge deserve further study.

Domain-invariant representation based methods often achieve domain invariance by aligning distributions of different domains as they did in domain adaptation. Li {\it et al} \cite{mmd-aae-cvpr-2018} impose MMD to an adversarial autoencoder to align the marginal distributions $P(Z)$ among domains, and the aligned distribution is matched with a pre-defined prior distribution by adversarial training. 
Motiian {\it et al} \cite{ccsa-iccv-2017} try to align the class-conditional distributions $P(Z|Y)$ for finer alignment. 
However, class-conditional distributions alignment based methods hardly deal with the domain shift caused by the label shift, which requires that categorical distribution $P(Y)$ remains unchanged among domains. Another important branch attempts to achieve domain-invariant representation via domain-invariant classifier learning. IRM \cite{irm-2019-arxiv} tries to learn a domain-invariant classifier by constraining that the classifier is simultaneously optimal for all domains. But this optimization problem is hard to solve. 
Our method CMCL learns domain-invariant classifier via posterior distribution alignment, an effective alternating optimization strategy is proposed to solve our optimization problem leading to excellent performance. 
Zhao {\it et al} \cite{dger-nips-2020} propose an entropy regularization term to align posterior distributions. According to our analysis, the proposed entropy term is a side-effect of minimizing KL-divergence, severely damaging classification performance. In our method, a term of maximum in-domain likelihood is proposed to eliminate this side-effect.

\section{Proposed Method} \label{section:formulation}
\subsection{Overview}
{In this paper, we focus on domain generalization for image classification. Suppose the sample and label spaces are represented by ${\mathcal{X}}$ and ${\mathcal{Y}}$ respectively, then a domain can be represented by a joint distribution defined on ${\mathcal{X}} \times {\mathcal{Y}}$. There are $N$ datasets $\mathcal{D} = \{\mathcal{S}^i = \{(x_j^i, y_j^i)\}_{j=1}^{M_i}\}_{i=1}^N$ sampled from domains with different distributions $\{ P^i(X, Y) \}_{i=1}^N$, where $M_i$ denotes the number of samples of dataset $\mathcal{S}^i$, $X\in \mathcal{X}$ and $Y \in \mathcal{Y}$. }
Let $P(X, Y)$ denote the ground-truth joint distribution in the real world. As shown in Figure \ref{fig:domains}, we suppose that $P(X, Y)$ yields distributions of training domains $\{ P^i(X, Y) \}_{i=1}^N$ and distribution of unseen domain $P^u(X, Y)$, with different domain shift due to different selection bias.
\begin{figure}[!htbp]
\setlength{\abovecaptionskip}{0pt}
\setlength{\belowcaptionskip}{0pt}
\renewcommand{\figurename}{Figure}
\centering
\includegraphics[width=0.3\textwidth]{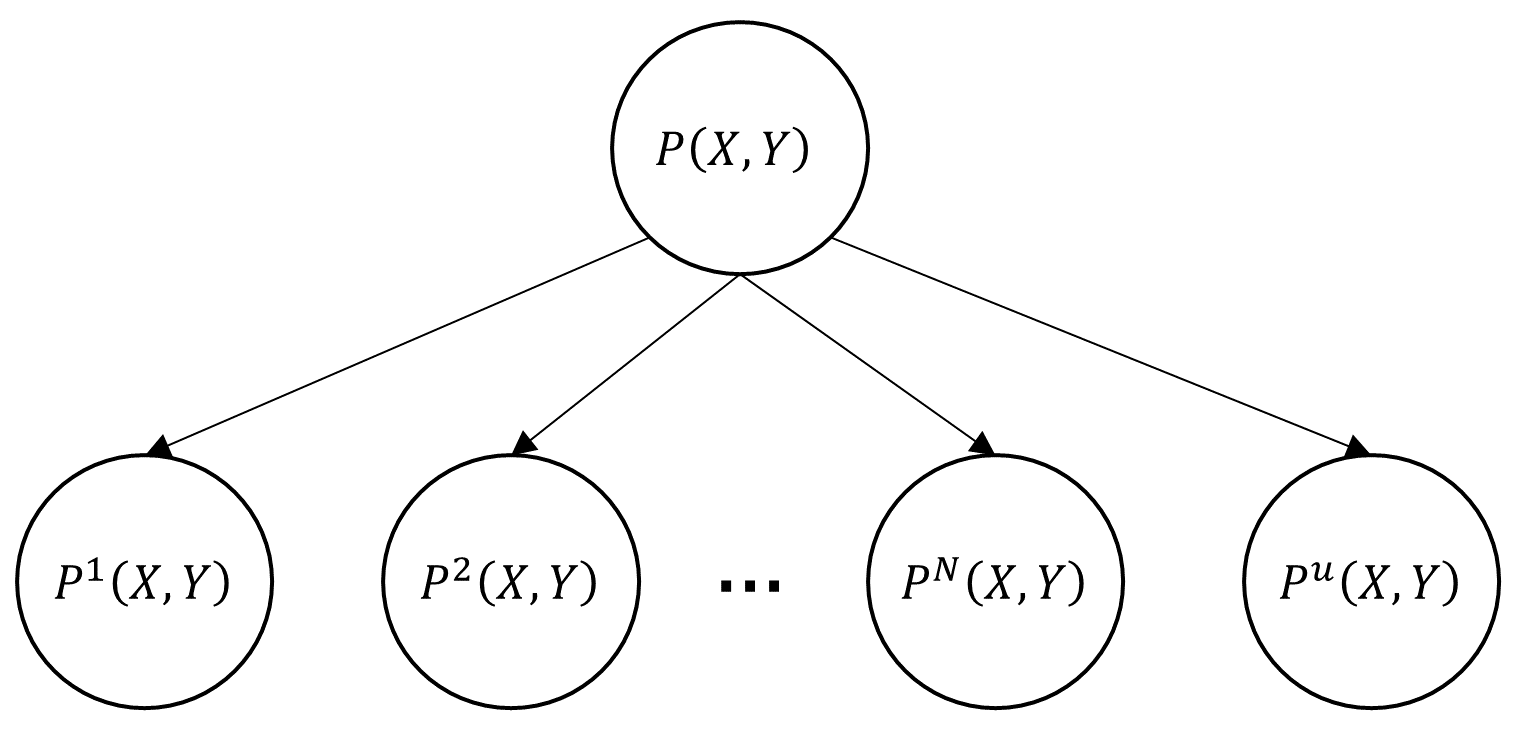}
\caption{Illustration of the generation process of domain-specific distributions \cite{bound-2021-arxiv}.}
\label{fig:domains}
\end{figure}

Given several training domains following different distributions, domain generalization aims to learn a model which is expected to overcome the domain shift and maintain its performance on unseen domains.
{
In order to overcome the distribution shift across domains, we try to learn a {\it domain-invariant representation space} in which the joint distributions of different domains are aligned. 

\begin{definition}[\textbf{Domain-Invariant Representation}]
	\label{def:dir} 
	Let $\mathcal{E}$ be a set of all possible domains. $F(\cdot): {\mathcal{X}} \times {\mathcal{Y}} \rightarrow \mathbb{R}^d$ is a feature mapping function that transforms raw input to the domain-invariant representation space. A representation space is domain-invariant if
	\begin{equation}\label{fml:dir}
		\forall{i\neq j \in \mathcal{E}} \quad P^i\left(Z, Y\right) = P^j\left(Z, Y\right)
	\end{equation}
	where $Z=F(X)$.
\end{definition}

To obtain the domain-invariant representation space, we firstly focus on aligning the posterior distribution from the perspective of {\it domain-invariant classifier} learning.

\begin{definition}[\textbf{Domain-Invariant Classifier}]
	\label{def:CMCL} Given a particular representation space, a domain-invariant classifier is simultaneously Bayes optimal classifier on any domain, which can be obtained when posterior distributions of different domains are aligned:
	\begin{equation}\label{fml:CMCL}
		\forall{i\neq j \in \mathcal{E}} \quad P^i\left(Y|Z\right) = P^j\left(Y|Z\right)
	\end{equation}
\end{definition}

We propose an optimization problem to learn the domain-invariant classifier, which minimizes the KL-divergence between posterior distributions of different domains and maximizes the discrimination of the in-domain feature space (see Section \ref{sec:Maximum In-Domain Likelihood}). The optimization objective is formalized as an expectation of the KL-divergence computed on ground-truth marginal distribution $P(Z)$ to enhance the generalizability of the domain-invariant classifier on unseen domains. Due to the unavailability of ground-truth marginal distribution $P(Z)$, a marginal distribution alignment constraint is proposed to approximate $P(Z)$ by $\{P^i(Z) \}_{i=1}^N$ under a convex hull assumption. Finally, a constrained maximum cross-domain likelihood optimization problem is formalized (see Section \ref{sec:Marginal Distribution Alignment Constraint}). Joint distributions are naturally aligned after solving this constrained optimization problem.
An alternating optimization strategy is proposed to solve this constrained optimization problem (see Section \ref{sec:aos}). The pipeline of the proposed alternating optimization process is illustrated in Figure \ref{fig:CMCL}.

\subsection{Constrained Maximum Cross-Domain Likelihood}\label{sec:objective formalization}
The typical classifier of traditional deep learning assumes that samples follow independent identically distribution and  tries to minimize the following objective:
\begin{equation}\label{fml:erm}
	\begin{aligned}
		\mathop{\min}_F\  - \sum_{i=1}^N \mathbb{E}_{P^i\left(Z,Y\right)}\left[\log P^g\left(Y|Z \right) \right]
	\end{aligned}
\end{equation}
where $Z=F(X)$ denotes the feature of raw input $X$, $P^g(Y|Z)$ denotes the global classifier trained with the data in all source domains. $\mathbb{E}_{P(Z, Y)}[\cdot]$ denotes an expectation computed over the distribution $P(Z, Y)$, {\it i.e.}, $\int P(Z,Y)[\cdot]dZdY$. Eq. (\ref{fml:erm}) is a regular empirical risk and can be regarded as a term of maximum likelihood, which ignores the fact that the data collected from different environments (domains) generally present distribution shift. 

In order to learn a domain-invariant classifier with better generalization ability on unseen domains, in this study,  we propose to minimize the KL-divergence between posterior distributions of different domains as follows:
\begin{equation}\label{fml:klobjective}
	\begin{aligned}
		\mathop{\min}_F\ &\mathbb{E}_{P\left(Z\right)}\left[\sum_{i \neq j} \text{KL}\left(P^i\left(Y|Z\right) \parallel P^j\left(Y|Z\right)\right) \right]\\
		& - \sum_{i=1}^N \mathbb{E}_{P^i\left(Z,Y\right)}\left[\log P^g\left(Y|Z \right) \right]
	\end{aligned}
\end{equation}
where $P\left(Z\right)$ denotes the ground-truth marginal distribution in the real world. 
The first term of the above formula means that a representation space is optimized hoping that all domain-specific posterior distributions can be the same for any given sample sampled from the ground-truth marginal distribution. 

Note that the expectation calculated on the ground-truth marginal distribution makes the optimization objective more general instead of being limited to source domains. If the expectation is calculated on source domains, the alignment of posterior distribution can only be achieved on limited several source domains. To generalize to unseen domains, the ideal optimization object should be an expectation calculated on unseen test distributions. 
An unseen test domain, which is yielded from the ground-truth distribution with a selection bias leading to domain shift, is naturally near to the ground-truth distribution. The distribution shift between unseen test domain and the ground-truth marginal distribution may be small than that between unseen domain and source domains. So the ground-truth marginal distribution is a reasonable substitute for arbitrary unseen test distributions, and hope that the learned classifier can generalize well to unseen test domains. 

\subsubsection{Maximum In-Domain Likelihood}\label{sec:Maximum In-Domain Likelihood}
However, minimizing the KL-divergence directly would produce a side-effect that can seriously damage the classification performance. To illustrate this more clearly, we divide KL-divergence  into two terms as follows:
\begin{equation}\label{fml:regklobjective}
	\begin{aligned}
		&\text{KL}\left(P^i\left(Y|Z\right) \parallel P^j\left(Y|Z\right)\right)   \\
		=& \underbrace{\mathbb{E}_{P^i\left(Y|Z\right)}\left[\log P^i\left(Y|Z\right) \right]}_\text{Negative\ Entropy}
		- \mathbb{E}_{P^i\left(Y|Z\right)} \left[\log  P^j\left(Y|Z\right)\right]
	\end{aligned}
\end{equation}	
When minimizing the KL-divergence, the first term is also minimized, which is essentially maximum entropy. 
Greater entropy means greater prediction uncertainty, which is contrary to the goal of the classification task. 
To solve this problem, another optimization objective is proposed:
\begin{equation}\label{fml:original-objective}
	\begin{split}
		\mathop{\min}_F \sum_{i \neq j} \Big( &\mathbb{E}_{P\left(Z\right)}\left[\text{KL}\left(P^i\left(Y|Z\right) \parallel P^j\left(Y|Z\right)\right)\right]  \\
		&\underbrace{-  \mathbb{E}_{P^i\left(Z,Y\right)}\left[\log P^i\left(Y|Z\right)\right]}_\text{Maximum In-Domain Likelihood}\Big)
		\\
		- \sum_{i=1}^N & \mathbb{E}_{P^i\left(Z,Y\right)}\left[\log P^g\left(Y|Z \right)\right]
	\end{split}
\end{equation}
A new term is proposed, which maximizes the posterior probability of the labeled data $(Z, Y)$ sampled from each domain. This term aims to maintain the discrimination of the learned representation space. Actually, it is essentially a maximum in-domain likelihood objective. This term is obviously different from the third term, which is a maximum global likelihood objective. The former measures the in-domain likelihood on domain-specific distributions, while the latter measures the global likelihood on the global distribution by aggregating images from all source domains.
Next, we introduce the following necessary definition for further analyzing the optimization problem in Eq.(\ref{fml:original-objective}).
\begin{definition}[\textbf{Marginal Distribution Error}]
	\label{def:ef} 
	In the representation space, let $P(Z)$ be the ground-truth marginal distribution. For the marginal distribution $P^i(Z)$ in each source domain, $0\le i \le N$, there exists an distribution error  $\Delta^i(Z)$ such that $\Delta^i(Z) = P(Z) - P^i(Z)$.
\end{definition}
The formulation in Eq.(\ref{fml:original-objective}) can be further decomposed as:
\begin{equation}\label{fml:objective}
	\begin{aligned}
		& \sum_{i \neq j} \Big(- \mathbb{E}_{P^i\left(Z,Y\right)} \left[\log P^j\left(Y|Z\right) \right]\\
		&\quad \quad \ + \int \Delta^i(Z) \text{KL}\left(P^i\left(Y|Z\right) \parallel P^j\left(Y|Z\right)\right) dZ\Big) \\
		&- \sum_{i=1}^N \mathbb{E}_{P^i\left(Z,Y\right)} \left[\log P^g\left(Y|Z \right)\right] \\
	\end{aligned}
\end{equation}
We provide the detailed derivation of Eq.(\ref{fml:objective}) in supplementary materials.
As shown above, the proposed new term of maximum in-domain likelihood eliminates the side-effect of minimizing KL-divergence. Original optimization objective in Eq.(\ref{fml:original-objective}) is transformed into a new form in Eq.(\ref{fml:objective}). 

\subsubsection{Marginal Distribution Alignment Constraint}\label{sec:Marginal Distribution Alignment Constraint}
Due to the unavailability of ground-truth marginal distribution, there is no way to optimize the integral term $ \int \Delta^i(Z) \text{KL}\left(P^i\left(Y|Z\right) \parallel P^j\left(Y|Z\right)\right) dZ\Big)$ in Eq.(\ref{fml:objective}) directly. Hence we introduce a new reasonable assumption which is critical for distribution alignment based domain generalization.
\begin{assumption}[\textbf{Inner Point of the Convex Hull}]
	\label{ass:mda}
	Let a set of marginal distributions of source domains in representation space be denoted as $\mathcal{M} = \{ P^i(Z)\}_{i=1}^N$. The convex hull of the set $\mathcal{M}$ is a set of all convex combinations of distributions in $\mathcal{M}$:
	\begin{equation}
		\Lambda = \left\{\sum_{i}^N \pi_i P^i(Z) \left\vert P^i(Z) \in \mathcal{M} , \pi_i\ge0, \sum_{i=1}^N\pi_i=1\right\}\right. \\
	\end{equation}
	The ground-truth marginal distribution is always a inner point of the convex hull:
	\begin{equation}
		P(Z) \in \Lambda
	\end{equation}
\end{assumption}
\begin{figure}[t]
	\setlength{\abovecaptionskip}{0pt}
	\setlength{\belowcaptionskip}{0pt}
	\renewcommand{\figurename}{Figure}
	\centering
	\includegraphics[width=0.35\textwidth]{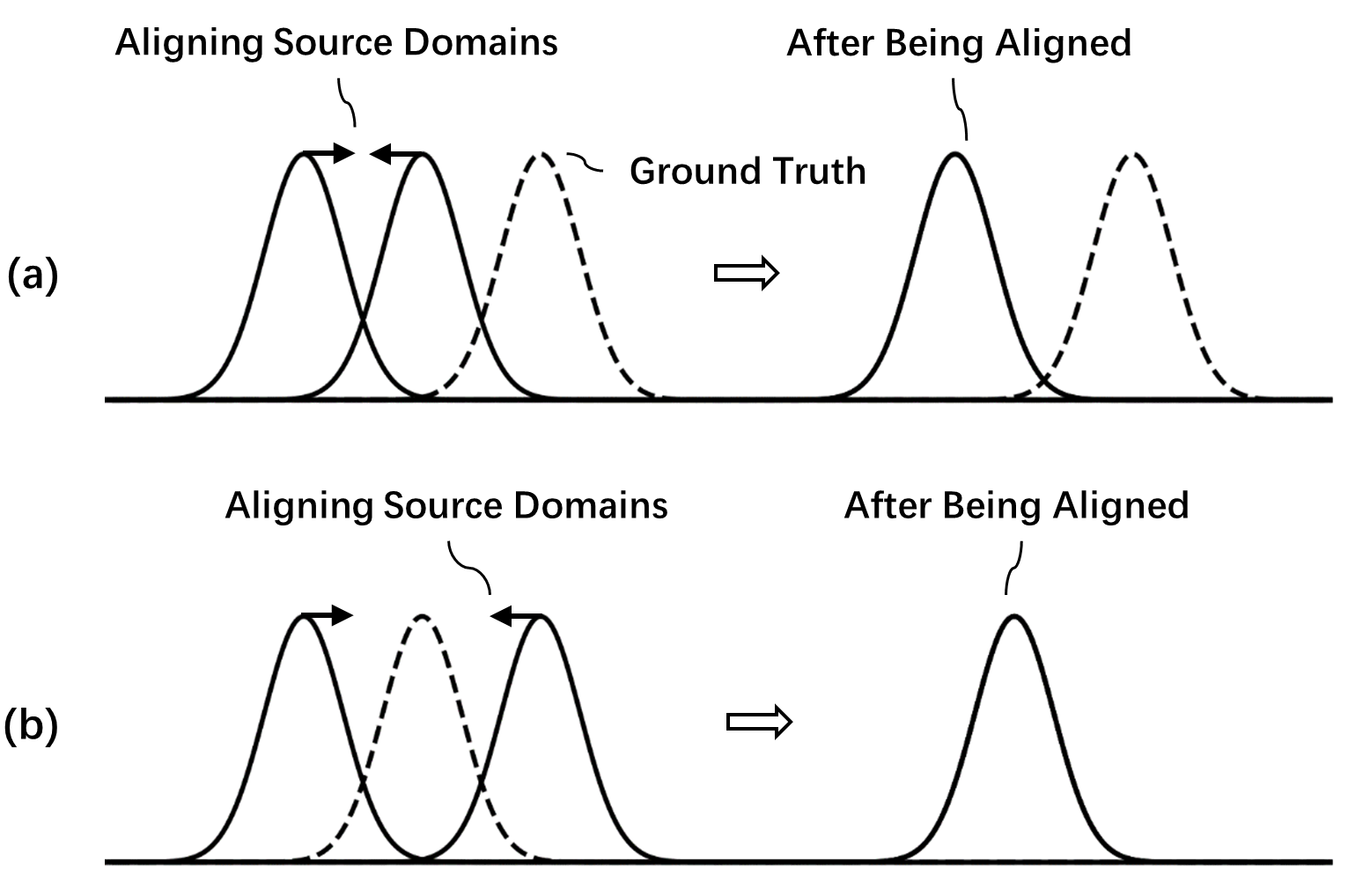}
	\caption{A toy example illustrates the rationality of Assumption \ref{ass:mda}. (a) The ground-truth marginal distribution lies outside of the convex hull of the source domains; (b) The ground-truth marginal distribution lies inside of the convex hull.}
	\label{fig:convex hull}
\end{figure}
As shown in Figure \ref{fig:convex hull}, it is reasonable that the ground-truth marginal distribution should lie inside of the convex hull of source domains for domain generalization. Under this assumption, the ground-truth marginal distribution can be depicted by source domains. Otherwise, the generalization on any possible unseen domain given several source domains can not be guaranteed, and domain generalization would be an unattainable goal. Similar assumptions are also covered in \cite{dro-arxiv-2019, bound-2021-arxiv, iedg-2022-aistats, re-icml-2021}. \cite{dro-arxiv-2019} and \cite{bound-2021-arxiv} assume that the distributions of the unseen domain stay inside the convex hull of source domains. \cite{dro-arxiv-2019} tries to optimize the worst-case expected loss over an uncertainty set of distributions, which encodes the possible test distributions. The uncertainty set is defined as a set of convex combinations of source domains. Even though \cite{iedg-2022-aistats} and \cite{re-icml-2021} try to handle scenarios that  unseen domains are extrapolations of source domains, they still admit that many existing researches are based on the basic assumption that unseen domains can be seen as interpolations of source domain and it is an important scenario for consideration. 

Under the above assumption, we try to align marginal distributions across different source domains so that the convex hull shrinks to a certain point. In this case, the ground-truth marginal distribution would be aligned to domain-specific marginal distributions, and the integral term in Eq.(\ref{fml:objective}) would approach 0. In other words, we hope that $\Delta^i(Z)$ is negligibly small after aligning $\{P^i(Z)\}^{N}_{i=1}$. We can get the following proposition by adding a constraint to Eq.(\ref{fml:objective}):
\begin{proposition}[\textbf{Constrained Maximum Cross-Domain Likelihood}]
	\label{pro:mcdl}
	Under Assumption \ref{ass:mda}, if the marginal distributions of source domains are aligned, the original optimization objective in Eq.(\ref{fml:original-objective}) can be achieved by solving the following constrained optimization problem:
	\begin{equation}\label{final optimization objective}
		\begin{aligned}
			&\mathop{\max}_F\ \sum_{i \neq j} \mathbb{E}_{P^i\left(Z,Y\right)}\left[\log P^j\left(Y|Z\right)\right]\\
			&\quad \ \quad  + \sum_{i=1}^N \mathbb{E}_{P^i\left(Z,Y\right)} \left[\log P^g\left(Y|Z \right) \right] \\
			&s.t. \quad 
			\forall 1\le i \neq j\le N, \quad P^i(Z) = P^j(Z)
		\end{aligned}
	\end{equation}	
\end{proposition}
\begin{proof}
	Under Assumption \ref{ass:mda}, if $\forall i \neq j, \ P^i(Z) = P^j(Z)$, we can get that $\forall i, \ P(Z) = P^i(Z)$, and then $\Delta^i(Z)\equiv 0$. Hence $\int \Delta^i(Z) KL\left(P^i\left(Y|Z\right) \parallel P^j\left(Y|Z\right)\right) dZ=0$, and then we can get that Eq.(\ref{final optimization objective}) is equivalent to Eq.(\ref{fml:original-objective}) 
\end{proof}
Optimizing both KL-divergence and maximum in-domain likelihood generates a constrained optimization problem, containing a term of maximum cross-domain likelihood under the condition of marginal distribution alignment, which means that the data sampled from one domain should have high posterior probability even though measured in the posterior distribution of another domain. This optimization objective of maximum cross-domain likelihood realizes the alignment of posterior distributions while improving the discrimination of representation space, and extends the traditional maximum likelihood to the domain shift setting. 
Marginal distributions and posterior distributions in the representation space will be aligned by solving this constrained optimization problem, and thus joint distributions will be aligned naturally.
Furthermore, the marginal distribution alignment is non-trivially coupled with posterior distribution alignment, which is indeed designed for the purposed of enhancing the generalization ability of the domain-invariant classifier.

\begin{figure*}[t]
	\setlength{\abovecaptionskip}{0pt}
	\setlength{\belowcaptionskip}{0pt}
	\renewcommand{\figurename}{Figure}
	\centering
	\includegraphics[width=0.99\textwidth]{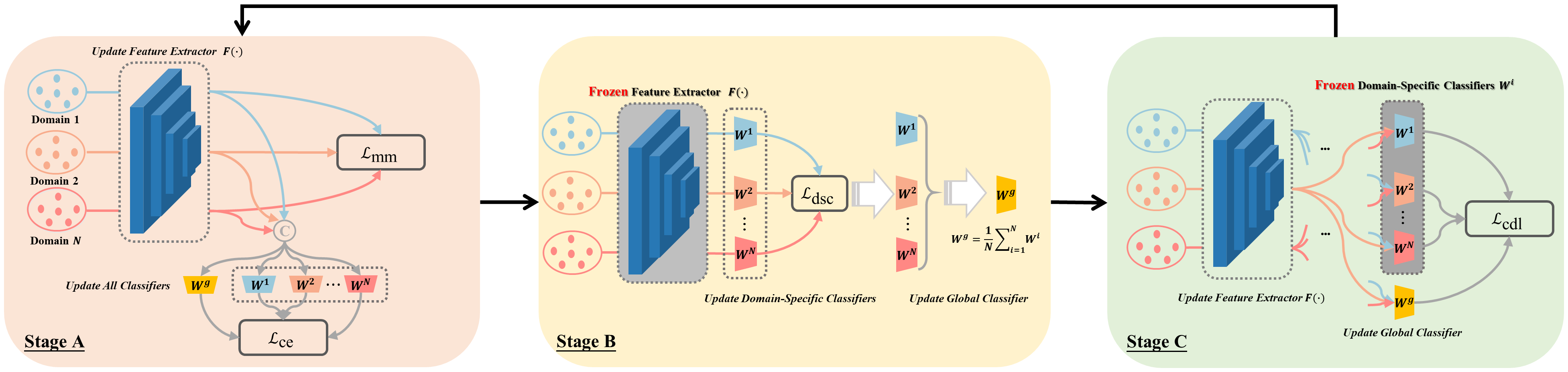}
	\caption{The flowchart of the proposed alternating optimization strategy. Multiple domain-specific classifiers are prepared for posterior distribution alignment, along with a global classifier used at test time. Stage A: Train all classifiers via vanilla empirical risk minimization, and align marginal distributions. Stage B: Estimate the posterior distribution of each domain in a fixed representation space, that is, the feature extractor is frozen. Stage C: Update the feature extractor to align the posterior distributions parameterized by frozen domain-specific classifiers via maximum cross-domain likelihood.}
	\label{fig:CMCL}
\end{figure*}

\subsubsection{The Practical Operation}\label{sec:Practical}
The non-convex constrained optimization problem described in Eq.(\ref{final optimization objective})  is hard to be solved. For simplicity, we transform it into an unconstrained optimization problem by adding a penalization term:
\begin{equation}\label{real final optimization objective}
	\begin{aligned}
		\mathop{\max}_F \sum_{i \neq j} \Big(& \mathbb{E}_{P^i\left(Z,Y\right)}\left[\log P^j\left(Y|Z\right)\right] 
		 - \lambda \text{Dis}\left(P^i(Z), P^j(Z) \right)\Big)\\
		+ \sum_{i=1}^N & \mathbb{E}_{P^i\left(Z,Y\right)} \left[\log P^g\left(Y|Z \right)\right] \\
	\end{aligned}
\end{equation}	
where $\lambda$ is a parameter controlling the intensity of the penalization term, and $\text{Dis}(\cdot)$ denotes the distance between two distributions. We adopt the moment matching loss \cite{{cmd-iclr-2017}} to implement the penalization term $\text{Dis}(\cdot)$. The first-order raw moment and second-order central moment of marginal distributions can be calculated as follows:
\begin{equation}\label{fml:mean}
	\begin{aligned}
		\bar{z}^i = \frac{1}{|\mathcal{S}^i|} \sum_{x \in \mathcal{S}^i} F\left(x\right)
	\end{aligned}
\end{equation}
\begin{equation}\label{fml:cov}
	\begin{aligned}
		C^i = \frac{1}{|\mathcal{S}^i|-1} \sum_{x \in \mathcal{S}^i} \left(F\left(x\right) - \bar{z}^i\right)\left(F\left(x\right) - \bar{z}^i\right)^T
	\end{aligned}
\end{equation}
Moment matching loss functions are designed as:
\begin{equation}\label{fml:loss mean}
	\begin{aligned}
		\mathcal{L}_\text{mean} = \frac{2}{N(N-1)d}\sum_{i \neq j}\left\Vert \bar{z}^i - \bar{z}^j\right\Vert_F^2
	\end{aligned}
\end{equation}
\begin{equation}\label{fml:loss cov}
	\begin{aligned}
		\mathcal{L}_\text{cov} = \frac{2}{N(N-1)d^2}\sum_{i \neq j} \left\Vert C^i - C^j \right \Vert_F^2
	\end{aligned}
\end{equation}
where $d$ denotes the dimension of features used to rescale the loss value, and $\Vert \cdot \Vert_F^2$ denotes the squared matrix Frobenius norm.

Then the final moment matching loss function can be defined as:
\begin{equation}\label{mm}
	\begin{aligned}
		\mathcal{L}_\text{mm} = \lambda_1 \mathcal{L}_\text{mean} + \lambda_2 \mathcal{L}_\text{cov} 
	\end{aligned}
\end{equation}
where $\lambda_1$ and $\lambda_2$ are trade-off parameters.

\subsection{Alternating Optimization Strategy}\label{sec:aos}
In this subsection, we propose an alternating optimization strategy to approximately solve Eq.(\ref{real final optimization objective}). In this elaborately designed optimization process, the posterior distribution estimation and the posterior distribution alignment are decoupled and carried out at different stages, and the difference among domains is explicitly explored and then  minimized effectively.

\subsubsection{Parameterization of Posterior Distribution}
It is primary to calculate the posterior probability given a sample $P(Y|Z=z)$ when optimizing the objective of maximum cross-domain likelihood. We adopt the softmax classifier to parameterize the posterior distribution.

\begin{equation}\label{fml:softmax}
	\begin{aligned}
		P\left(Y=y | Z=z\right) = \frac{\exp\left(w_yz\right)}{\sum_{y\prime\in \mathcal{Y}} \exp{\left(w_{{y}\prime}z\right)}}
	\end{aligned}
\end{equation}
where $w_y$ and $w_{{y}\prime}$ denote the corresponding row of the parameter matrix $W \in \mathbb{R}^{K\times d}$ of the softmax classifier and $K$ is the number of classes. In the process of optimizing the maximum cross-domain likelihood objective described in Eq.(\ref{final optimization objective}) and Eq.(\ref{real final optimization objective}), posterior distributions of all domains need to be estimated separately. Hence $N$ domain-specific classifiers $\{W^i\}_{i=1}^N$ are introduced to parameterize the posterior distribution of each domain. 
In addition to domain-specific classifiers, we need to train a global classifier $W^g$ with all samples based on the learned representation, which is required by the second term in Eq.(\ref{real final optimization objective}).
\subsubsection{Alternating Optimization}\label{sec:AO}
In order to maximize cross-domain likelihood, we should estimate the posterior distributions of all domains $\{P^i(Y|Z; W^i)\}_{i=1}^{N}$ before updating the feature extractor $F(\cdot)$. After $F(\cdot)$ is updated, the representation space has been changed and $\{P^i(Y|Z; W^i)\}_{i=1}^{N}$ need to be re-estimated. Therefore, an alternating optimization strategy for domain-invariant classifier learning is designed to approximately solve the constrained optimization problem:

\textbf{Stage A:} The feature extractor and all classifiers are jointly trained through vanilla empirical risk minimization to maintain the classification ability of the classifiers and further enhance the discrimination ability of the learned representation extracted by $F(\cdot)$ during the alternating training process. The loss function can be calculated as:
\begin{equation}\label{lce}
	\begin{aligned}
		\mathcal{L}_\text{ce} = - \frac{1}{\left|\mathcal{D}\right|} \sum_{(x, y) \in \mathcal{D}} \Big(\sum_{i=1}^N \log P^i\left(y|F\left(x\right);W^i\right) \\
		+ \log P^g\left(y|F\left(x\right); W^g\right) \Big)
	\end{aligned}
\end{equation}
which is essentially a cross-entropy loss function. Additionally, at this stage, the penalization term in Eq.(\ref{real final optimization objective}) is implemented by aligning marginal distributions by moment matching. The loss function at this training stage can be defined as:
\begin{equation}\label{lcemm}
	\begin{aligned}
		\mathcal{L}_\text{cemm} = \mathcal{L}_\text{ce} + \mathcal{L}_\text{mm}
	\end{aligned}
\end{equation}


\textbf{Stage B:} The feature extractor is frozen, providing a deterministic representation space for estimating the posterior distributions, which is denoted by $\overline{F}(\cdot)$. Given the fixed representations, the domain-specific classifiers are trained with data sampled from respective domains. The loss function at this training stage can be defined as:
\begin{equation}\label{dsc}
	\begin{aligned}
		\mathcal{L}_\text{dsc} = - \sum_{i=1}^N \frac{1}{\left|\mathcal{S}^i\right|} \sum_{(x, y) \in \mathcal{S}^i} \log P^i\left(y|\overline{F}\left(x\right);W^i\right)
	\end{aligned}
\end{equation}
As mentioned earlier, the domain-specific classifiers tend to be consistent as the alternating training goes on. Then the optimal global classifier can be obtained at the convergence point of the domain-specific classifiers. Hence, at this stage, we set the parameters of the global classifier as the mean of all domain-specific classifiers to accelerate the convergence of the training process and improve the stability of the training process:
\begin{equation}\label{cls_mean}
	\begin{aligned}
		W^g = \frac{1}{N}\sum_{i=1}^N W^i
	\end{aligned}
\end{equation}

\textbf{Stage C:} The domain-specific classifiers is frozen, providing measurements of the posterior distributions for updating the feature extractor. Given the fixed domain-specific classifiers $\{\overline{W}^i\}_{i=1}^N$, the data sampled from one domain are fed into the classifier of another domain. Then the cross-domain likelihood is maximized by updating the feature extractor. The loss function at this training stage can be defined as:
\begin{equation}\label{cdl}
	\begin{aligned}
		\mathcal{L}_\text{cdl} = - \sum_{i=1}^N \frac{1}{\left|\mathcal{S}^i\right|} \sum_{(x, y) \in \mathcal{S}^i} \Big(&\sum_{j\neq i}\log P^j(y|F(x);\overline{W}^j)\\
		&	+ \log P^g(y|F(x);{W}^g) \Big)
	\end{aligned}
\end{equation}
At this stage, the initial point of parameters of the global classifier $W^g$ is the average of all domain-specific classifiers as mentioned at stage B. The global classifier is trained together with the feature extractor alleviating the problem of over-adjustment when maximizing cross-domain likelihood.

\begin{algorithm}[t]
	\SetAlgoLined
	\caption{\textbf{Alternating optimization strategy for domain-invariant classifier learning}}
	\label{algo:CMCL}
	\KwIn{Source domains: $\{\mathcal{S}^i = \{(x_j^i, y_j^i)\}_{j=1}^{N_i}\}_{i=1}^N$. \\
		\quad \quad \quad Online model: $F_o(\cdot)$, $\{W_o^i\}_{i=1}^N$ and $W_o^g$.\\
		\quad \quad \quad Target model: $F_t(\cdot)$ and $W_t^g$. \\
		\quad \quad \quad Paramenter: $\alpha$, $\lambda_1$, $\lambda_2$, $n$, $n_A$, $n_B$ and $n_C$. }
	\KwOut{Target model: $F_t(\cdot)$ and $W_t^g$}
	\BlankLine
	\For{$ite = 1 : n$}
	{
		// \textit{Stage A}\\
		\For{$ite_A = 1 : n_A$}
		{
			Sample a mini-batch data from each domain;\\
			Calculate $\mathcal{L}_\text{ce}$ according to (\ref{lcemm});\\
			Calculate $\mathcal{L}_\text{mm}$ according to (\ref{mm});\\
			Update $F_o(\cdot)$, $\{W_o^i\}_{i=1}^N$ and $W_o^g$ by gradient descent;\\
			Update $F_t(\cdot)$ and $W_t^g$  according to (\ref{mu});\\
		}
		// \textit{Stage B}\\
		\For{$ite_B = 1 : n_B$}
		{
			Freeze the parameters of feature extractor $F_o(\cdot)$;\\
			Sample a mini-batch data from each domain;\\
			Calculate $\mathcal{L}_\text{dsc}$ according to (\ref{dsc});\\
			Update $\{W_o^i\}_{i=1}^N$ by gradient descent;\\
			Update $W_o^g$ according to (\ref{cls_mean})
		}
		// \textit{Stage C}\\
		\For{$ite_C = 1 : n_C$}
		{
			Freeze the parameters of domain-specific classifiers $\{W_o^i\}_{i=1}^N$;\\
			Sample a mini-batch data from each domain;\\
			Calculate $\mathcal{L}_\text{cdl}$ according to (\ref{cdl});\\
			Update $F_o(\cdot)$ and $W_o^g$ by gradient descent;\\
			Update $F_t(\cdot)$ and $W_t^g$ according to  (\ref{mu});\\
		}				
	}
\end{algorithm}

As described above, we carry out three stages of the training process alternately and this process keeps cycling.
To improve the stability of the training process and facilitate generalization, in addition to the online model which is updated along the negative gradient direction, we introduce an extra target model which is updated along the differential direction of the parameters between the online model and target model. It is essentially the Exponential Moving Average (EMA) of parameters of the online model:
\begin{equation}\label{mu}
	\Theta_{t}^{target} = \Theta_{t-1}^{target} + \alpha \left(\Theta_{t}^{online}-\Theta_{t-1}^{target}\right)
\end{equation}
where $\Theta = \{W^g, F\}$, $\Theta_{t}^{target}$ and $\Theta_{t}^{online}$ denote the parameters of target model and online model at step $t$ respectively, and $\alpha$ denotes the step size of EMA. In this paper, $\alpha$ is set to 0.001 for all experiments.

As Figure \ref{fig:CMCL} shows, we optimize $\mathcal{L}_\text{cemm}$, $\mathcal{L}_\text{dsc}$ and $\mathcal{L}_\text{cdl}$ alternately to align marginal distributions and posterior distributions so that the constrained optimization problem described in Eq.(\ref{final optimization objective}) can be solved approximately. In order to illustrate the training process clearly, the pseudo-code of our algorithm is provided in Algorithm \ref{algo:CMCL}.

\section{Experimental Results and Analysis} \label{section:experiment}
In this section, we conduct extensive experiments on four popular domain generalization datasets to validate the effectiveness of the proposed CMCL. 
Ablation studies and further analysis are carried out to analyze the characteristics of the proposed CMCL approach.
\subsection{Datasets and Settings}
\begin{itemize}
    \item {\it Digits-DG} \cite{digits-dg-AAAI-2020} is a digit recognition benchmark, which is consisted of four classical datasets {\it MNIST} \cite{mnist-1998}, {\it MNIST-M} \cite{mnist-m-2015}, {\it SVHN} \cite{svhn-2011} and {\it SYN} \cite{mnist-m-2015}. Each dataset is treated as a domain, and the domain shift across the four datasets is mainly reflected in font style, background and image quality. Each domain contains seven categories, and each class contains 600 images. The original train-validation split in \cite{digits-dg-AAAI-2020} is adopted for fair comparison.
    \item {\it PACS} \cite{pacs-iccv-2017} is an object recognition benchmark which is consisted of four domains namely {\it Photo}, {\it Art}, {\it Cartoon} and {\it Sketch}. The main domain shift is reflected in large discrepancy in image style. There are 9991 images and 7 classes in {\it PACS}. The original train-validation split provided by \cite{pacs-iccv-2017} is adopted.
    \begin{figure}[!htbp]
	\setlength{\abovecaptionskip}{0pt}
	\setlength{\belowcaptionskip}{0pt}
	\renewcommand{\figurename}{Figure}
	\centering
	\includegraphics[width=0.49\textwidth]{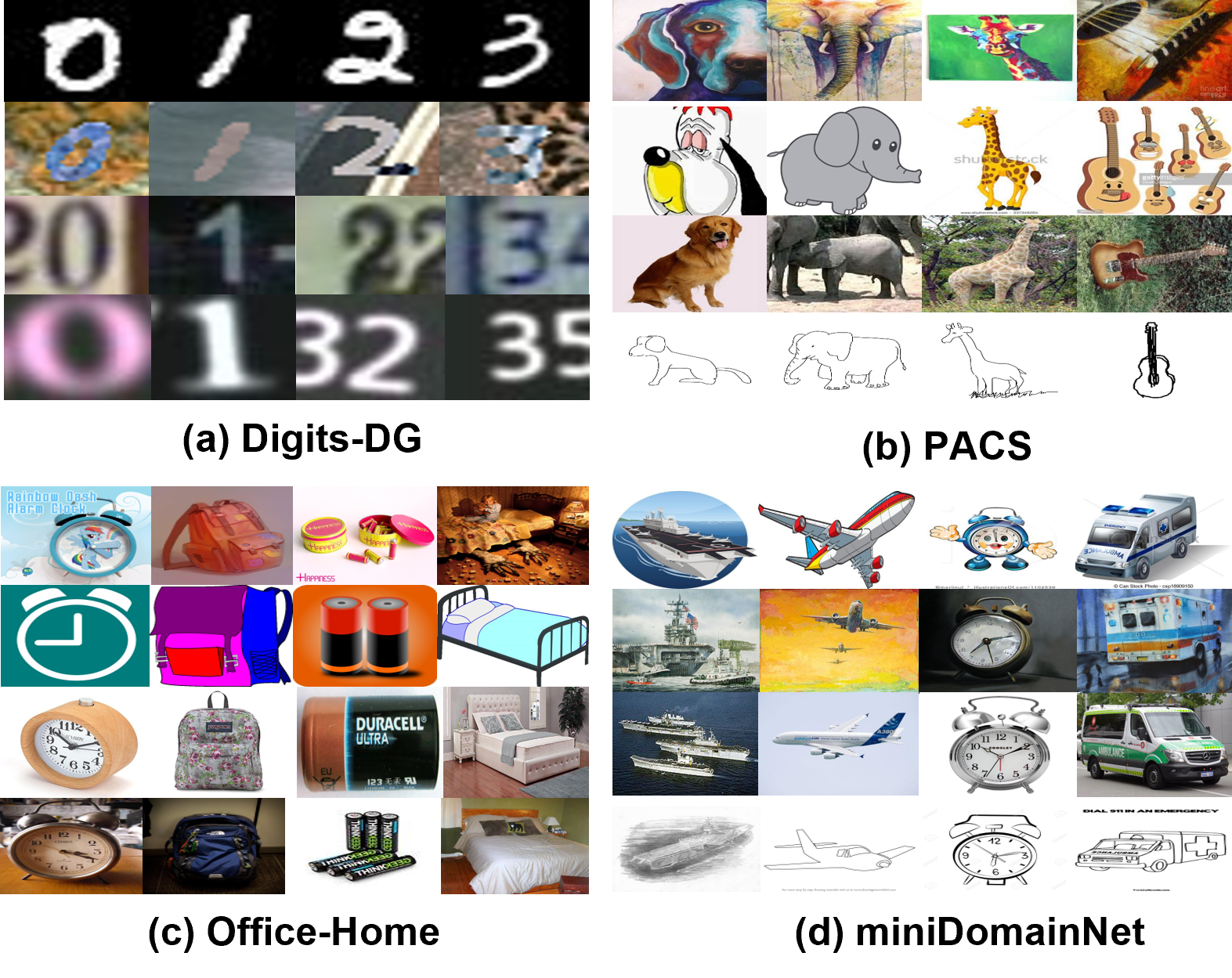}
	\caption{{Samples from PACS, Office-Home, miniDomainNet and Digits-DG datasets.}}
	\label{fig:datasets}
\end{figure}
    \item {\it Office-Home} \cite{office-home-cvpr-2017} is also an object recognition benchmark which consists of four domains, i.e., {\it Art}, {\it Clipart}, {\it Product} and {\it Real-world}. The main domain shift is reflected in image style and viewpoint, with less diversity than PACS. There are around 15,500 images in total, and each domain contains 65 classes. We randomly split each domain into 90\% for training and 10\% for validation following \cite{jgen-cvpr-2019}.
    \item {\it miniDomainNet} \cite{minidomainnet-tip-2021} is a larger object recognition benchmark, which takes a subset of {\it DomainNet} \cite{coral-iccv-2019} containing four domains namely {\it Clipart}, {\it Painting}, {\it Real} and {\it Sketch}. There are 140,006 images and 126 classes. The original train-validation split provided by \cite{minidomainnet-tip-2021} is adopted.
\end{itemize}

Images are resized to 32 $\times$ 32 for Digits-DG, and 224 $\times$ 224 for PACS, Office-Home and miniDomainNet. All experiments are conducted following the commonly used leave-one-domain-out protocol \cite{pacs-iccv-2017}. Specifically, one domain is specified as an unseen test domain, and the remaining domains are treated as source domains to train the model. To fairly compare with published methods, our models are trained using data only from the training split, and the models with the best result on the validation split are selected for testing. All results are reported based on the average top-1 classification accuracy over three repetitive runs.

\subsection{Implementation Details}

For all experiments, the models are constructed by a feature extractor, three domain-specific classifiers and a global classifier. The classifiers are implemented by softmax layers. The domain-specific classifiers are trained by an AdamW optimizer with a learning rate of 1e-5 and weight decay of 5e-4. The number of iterations of stage A, {\it i.e.}, $n_A$, is set to 1.

For Digits-DG, the feature extractor is constructed by four 3 $\times$ 3 conv layers (64 kernels), each followed by ReLU and 2 $\times$ 2 max-pooling, following \cite{digits-dg-AAAI-2020}. The feature extractor and global classifier are trained by SGD optimizer with learning rate of 0.05, batch size of 64 for each domain, momentum of 0.9 and weight decay of 5e-4. The number of outer loops and iteration of each stage, {\it i.e.}, $n$,  $n_B$ and $n_C$, are set to 4,000, 8, 6 respectively. In particular, $\lambda_1$ and $\lambda_2$ are set to 0.001, 0.01 for Digits-DG.
For PACS, Office-Home and miniDomainNet, ResNet-18 pre-trained on ImageNet is used as a feature extractor removing the last layer, and $\lambda_1$ and $\lambda_2$ are set to 10 and 100, respectively. Experiments using pre-trained ResNet-50 as feature extractor are also conducted on PACS and Office-Home. The feature extractor and global classifier are trained by AdamW optimizer with an initial learning rate of 1e-4, batch size of 64 for each domain and weight decay of 5e-4. 
We set $n=800$, $n_B=8$ and $n_C=6$ for PACS,  $n=1,200$, $n_B=6$ and $n_C=2$ for Office-Home, $n=2,000$, $n_B=6$ and $n_C=2$ for miniDomainNet.
Standard data augmentation including crops of random size, random horizontal flips, random color jitter and randomly converting the image to grayscale are used, following \cite{domainbed-iclr-2021}, \cite{stablnet-cvpr-2021}.

\subsection{Performance Comparison}
In this subsection, we compare our method with a series of domain generalization methods, presenting reported accuracy on each dataset. 
In existing domain generalization datasets, the domain shift is mainly reflected by the image style as shown in Figure \ref{fig:datasets}. Hence some works develop their methods based on the prior knowledge about domain shift, {\it e.g.}, the style randomization. For a fair comparison, methods are divided into two groups according to whether the prior knowledge of domain shift is used. Our method deals with domain generalization without the need of prior knowledge and can be applied to various datasets with more general domain shifts.

\begin{table}[t]
	\renewcommand{\arraystretch}{1.2}
	\centering
	{
		\caption{Performance comparison  on Digits-DG dataset.}
		\label{table:Digits-DG}
		\scalebox{1}{
			\begin{threeparttable}
				\begin{tabular}{l c c c c c }
					\toprule[1.5pt]
					\multicolumn{1}{l}{Dataset} & \multicolumn{5}{c}{Digits-DG} \\
					\midrule
					Domain  &   MNIST  &  MNIST-M   &  SVHN    &  SYN   &  Avg.   \\
					\midrule
					\multicolumn{6}{c} {\it{w/ Prior Knowledge about Domain Shift}} \\
					\midrule
					
					DDAIG \cite{digits-dg-AAAI-2020}	 & $96.60$   & $\underline{64.10}$          & $68.60$          & $81.00$          & $77.60$         \\
					
					MGFA \cite{mgfa-bmvc-2021}	  & $95.71$         & $60.66$          & $69.35$          & $74.38$          & $75.02$          \\
					MixStyle  \cite{mixstyle-iclr-2021}    & $96.50$     & $63.50$          & $64.70$    & $81.20$     & $76.50$      \\
					\midrule
					\multicolumn{6}{c} {\it{w/o Prior Knowledge about Domain Shift}} \\
					\midrule
					DeepAll \cite{digits-dg-AAAI-2020}   & $95.80$  & $58.80$ & $61.70$  & $78.60$ & $73.70$  \\
					CCSA \cite{ccsa-iccv-2017} & $95.20$   & $58.20$          & $65.50$          & $79.10$          & $74.50$ \\
					MMD-AAE	\cite{mmd-aae-cvpr-2018}	 & $96.50$   & $58.40$     & $65.00$          & $78.40$          & $74.60$ \\
					CrossGrad \cite{crossgrad-iclr-2018}	 & $\underline{96.70}$   & $61.10$     & $65.30$          & $80.20$          & $75.80$ \\
					L2A-OT \cite{l2a-ot-eccv-2020}  \ \    & $\underline{96.70}$      & ${63.90}$    & $68.60$          & $83.20$          & $78.10$         \\
					SFA-A  \cite{sfa-a-iccv-2021}    & $96.50$         & $\textbf{66.50}$          & $\underline{70.30}$          & $\underline{85.00}$          & $\underline{79.60}$    \\
					
					\midrule
					CMCL  & $\textbf{98.37}$ & $63.85$ & $\textbf{77.75}$ & $\textbf{94.03}$ & $\textbf{83.50}$   \\
					\bottomrule[1.5pt]
				\end{tabular}
				\footnotesize
				The best result is in bold face. Underlined ones represent the second-best results. 
	\end{threeparttable}}}
\end{table}

\subsubsection{Evaluation on Digits-DG}
The domain generalization results on the Digits-DG benchmark are reported in Table \ref{table:Digits-DG}. For all compared approaches, we summarize the results reported in their original papers.
We observe that the proposed CMCL achieves the best performance in average accuracy and significantly surpasses the second-best competitor by a large margin of 3.90\%. Specifically, CMCL outperforms competitors on MNIST, SVHN and SYN. The improvement on MNIST-M is not as significant as those on other domains, mainly due to its dissimilarity with source domains as shown at the second row in Figure \ref{fig:datasets}. On the contrary, image augmentation based methods, DDAIG, L2A-OT, and feature augmentation based method, SFA-A, obtain larger improvement on MNIST-M but perform worse than CMCL on other ones. Probably because the domain shift of MNIST, SVHN and SYN are independent of image style and texture, the proposed CMCL which does not rely on any prior knowledge of domain shift works better on these domains. Note that the proposed CMCL has great advantages over CCSA and MME-AAE, which are all domain-invariant representation based methods.

\begin{table}[t]
	\renewcommand{\arraystretch}{1.2}
	\centering
	{
		\caption{Performance comparison  on PACS dataset.}
		\label{table:PACS18}
		\scalebox{1}{
			\begin{threeparttable}
				\begin{tabular}{ l c c c c c  }
					\toprule[1.5pt]
					\multicolumn{1}{ l }{Dataset} & \multicolumn{5}{c }{PACS} \\
					\midrule
					\multicolumn{1}{ l }{Backbone} & \multicolumn{5}{c }{ResNet-18} \\
					\midrule
					Domain  &   Art  & Cartoon   &   Photo   &  Sketch    &   Avg.   \\
					\midrule
					\multicolumn{6}{c} {\it{w/ Prior Knowledge about Domain Shift}} \\
					\midrule
					
					DDAIG \cite{digits-dg-AAAI-2020}	 & ${84.20}$         & $78.10$      & $95.30$      & $74.70$          & $83.10$    \\
					DSON \cite{dson-eccv-2020}   & $\textbf{84.67}$         & $77.65$          & $95.87$          & $\underline{82.23}$          & $85.11$          \\
					MGFA \cite{mgfa-bmvc-2021}	  & $81.70$         & $77.61$          & $95.40$          & $76.02$          & $82.68$          \\
					MixStyle  \cite{mixstyle-iclr-2021}    & $84.10$         & $78.80$    & $96.10$     & $75.90$     & $83.70$           \\
					LDSDG   \cite{ldsdg-cvpr-2021}  & $81.44$    & $79.56$      & $95.51$     & $80.58$       & $84.27$          \\
					SagNet	\cite{sagnet-cvpr-2021} & $83.58$         & $77.66$      & $95.47$      & $76.30$          & $83.25$    \\
					pAdaIN \cite{padin-cvpr-2021}  & $81.74$         & $76.91$          & $96.29$          & $75.13$          & $82.51$    \\		
					EFDMix \cite{efd-cvpr-2022}    & $83.90$         & $79.40$          & $\textbf{96.80}$          & $75.00$          & $83.90$    \\		
					
					\midrule
					\multicolumn{6}{c} {\it{w/o Prior Knowledge about Domain Shift}} \\
					\midrule
					DeepAll \cite{dger-nips-2020}  & $78.93$  & $75.02$ & $96.60$  & $70.48$ & $80.25$  \\
					Metareg  \cite{metareg-nips-2018}	 & $83.70$         & $77.20$          & $95.50$          & $70.30$          & $81.70$  \\
					MASF \cite{masf-nips-2019}  & $80.29$          & $77.17$          & $94.99$          & $71.69$          & $81.04$         \\
					DMG \cite{dmg-eccv-2020}  & $76.90$          & $\underline{80.38}$          & $93.35$          & $75.21$          & $81.46$         \\
					EISNet  \cite{eisnet-eccv-2020} & $81.89$         & $76.44$          & $95.93$          & $74.33$          & $82.15$  \\
					DGER \cite{dger-nips-2020}  \ \    & $80.70$         & $76.40$          & ${96.65}$          & $71.77$     & $81.38$     \\
					MMLD  \cite{mmld-aaai-2020}   & $81.28$          & $77.16$          & $96.09$          & $72.29$          & $81.83$          \\
					L2A-OT  \cite{l2a-ot-eccv-2020}     & $83.30$         & $78.20$          & $96.20$          & $73.60$          & $82.80$   \\
					
					RSC \cite{rsc-eccv-2020}			& $83.43$         & $80.31$          & $95.99$          & $80.85$          & $85.15$    \\
					
					RSC* \cite{rsc-eccv-2020}			& $78.90$         & $76.88$          & $94.10$          & $76.81$          & $81.67$    \\
					
					
					NAS-OoD	\cite{nas-ood-iccv-2021}	 & $83.74$    & $79.69$      & $96.23$     & $77.27$       & $84.23$          \\
					SFA-A  \cite{sfa-a-iccv-2021}   & $81.20$         & $77.80$          & $93.90$          & $73.70$          & $81.70$    \\
					
					DAML	\cite{daml-cvpr-2021}	 & $83.00$         & $74.10$          & $95.60$          & $78.10$          & $82.70$    \\
					StableNet	\cite{stablnet-cvpr-2021}	 & $81.74$         & $79.91$      & $96.53$      & $80.50$          & $84.69$    \\
					
					MDGHybrid \cite{mdghybrid-icml-2021} & $81.71$         & $\textbf{81.61}$      & $\underline{96.67}$      & $81.05$          & $\underline{85.53}$    \\
					DSFG  \cite{dsfg-wacv-2022} & $83.89$         & $76.45$      & $95.09$      & $78.26$          & $83.42$    \\
					
					\midrule
					
					CMCL & $\underline{84.55}$   & $80.08$ & $94.95$ & $\textbf{82.86}$ & $\textbf{85.61}$ \\
					\bottomrule[1.5pt]
				\end{tabular}
				\footnotesize
				The best result is in bold face. Underlined ones represent the second-best results. RSC* denotes the reproduced results from pAdaIN \cite{padin-cvpr-2021}.
	\end{threeparttable}}}
\end{table}

\begin{table}[t]
	\renewcommand{\arraystretch}{1.2}
	\centering
	{
		\caption{Performance comparison  on PACS dataset.}
		\label{table:PACS50}
		\scalebox{1}{
			\begin{threeparttable}
				\begin{tabular}{l c c c c c }
					\toprule[1.5pt]
					\multicolumn{1}{l}{Dataset} & \multicolumn{5}{c }{PACS} \\
					\midrule
					\multicolumn{1}{l}{Backbone} & \multicolumn{5}{c }{ResNet-50} \\
					\midrule
					Domain  &   Art  & Cartoon   &   Photo   &  Sketch    &   Avg.   \\
					
					\midrule
					\multicolumn{6}{c} {\it{w/ Prior Knowledge about Domain Shift}} \\
					\midrule
					
					DSON \cite{dson-eccv-2020}    & $87.04$         & $80.62$          & $95.99$          & $82.90$          & $86.64$   \\
					MGFA \cite{mgfa-bmvc-2021}	  & $86.40$         & $79.45$          & $97.86$          & $78.72$          & $85.60$          \\
					pAdaIN \cite{padin-cvpr-2021}  & $85.82$         & $81.06$          & $97.17$          & $77.37$          & $85.36$    \\
					EFDMix \cite{efd-cvpr-2022}    & $\textbf{90.60}$         & $\underline{82.50}$          & $98.10$          & $76.40$          & $86.90$    \\
					
					\midrule
					\multicolumn{6}{c} {\it{w/o Prior Knowledge about Domain Shift}} \\
					\midrule
					DeepAll \cite{dger-nips-2020}  & $86.18$  & $76.79$ & $98.14$  & $74.66$ & $83.94$  \\	
					Metareg  \cite{metareg-nips-2018}	 & $87.20$         & $79.20$          & $97.60$          & $70.30$          & $83.60$  \\
					IRM  \cite{irm-2019-arxiv}                & $84.80$     & $76.40$     &$96.70$    & $76.10$       &$83.50$                 \\
					MASF \cite{masf-nips-2019} 	 & $82.89$         & $80.49$          & $95.01$          & $72.29$          & $82.67$  \\
					DMG \cite{dmg-eccv-2020}  & $82.57$          & $78.11$          & $94.49$          & $78.32$          & $83.37$		 \\
					EISNet \cite{eisnet-eccv-2020} & $86.64$         & $81.53$          & $97.11$          & $78.07$          & $85.84$  \\
					DGER \cite{dger-nips-2020}      & ${87.51}$         & $79.31$          & $\underline{98.25}$          & $76.30$          & $85.34$         \\
					RSC 	\cite{rsc-eccv-2020}		& $\underline{87.89}$         & $82.16$          & $97.92$          & $83.35$          & $\textbf{87.83}$    \\
					RSC* \cite{rsc-eccv-2020}			& $81.38$         & $80.14$          & $93.72$          & $82.31$          & $84.38$    \\
					MDGHybrid \cite{mdghybrid-icml-2021}  & $86.74$         & ${82.32}$      & $\textbf{98.36}$      & $82.66$          & $87.52$    \\
					DSFG \cite{dsfg-wacv-2022} & $87.30$         & $80.93$      & $96.59$      & $\underline{83.43}$          & $87.06$    \\
					\midrule
					
					CMCL  & ${87.57}$   & $\textbf{83.60}$ & $96.03$ & $\textbf{83.73}$ & $\underline{87.73}$ \\
					\bottomrule[1.5pt]
				\end{tabular}
				\footnotesize
				The best result is in bold face. Underlined ones represent the second-best results. RSC* denotes the reproduced results from pAdaIN \cite{padin-cvpr-2021}.
	\end{threeparttable}}}
\end{table}

\subsubsection{Evaluation on PACS}
Results on PACS with ResNet-18 and ResNet-50 are presented in Table \ref{table:PACS18} and Table \ref{table:PACS50}, respectively. For all competitors, we summarize the results reported in their original papers. 
We can observe that CMCL outperforms all comparison approaches on average accuracy with ResNet-18 as the feature extractor and obtains comparable performance with the reported best approach with ResNet-50 as feature extractor. The experiments on feature extractors of different sizes further prove the effectiveness of our CMCL. Specifically, CMCL achieves the best accuracy on Sketch and the second best accuracy on Art in Table \ref{table:PACS18} and obtains the best performance on Sketch and Cartoon in Table \ref{table:PACS50}.
We notice that there is a performance drop on Photo compared to the vanilla empirical risk minimization method DeepAll. It is probably because of the ImageNet pretraining. As explained in \cite{randconv-iclr-2020}, models pre-trained on ImageNet may be biased towards texture, and finetuning those models on PACS using empirical risk minimization may inherit this bias, thus leading to better performance on Photo which is similar to ImageNet.

\subsubsection{Evaluation on Office-Home}
Experiments results with ResNet-18 as the feature extractor are reported in Table \ref{table:Office-Home}. For all comparison methods, we summarize the results reported in their original papers. We also report the results with ResNet-50 in Table \ref{table:Office-Home50}. For all comparison methods in Table \ref{table:Office-Home50}, we summarize the results reported in a published work \cite{domainbed-iclr-2021}, which uses the same experiment settings as ours, including data augmentation, model selection and data splitting.
From Table \ref{table:Office-Home}, we can see that our method again achieves the best average accuracy compared to the competitors, though the domain discrepancy of Office-Home is less than other datasets, which is unfavorable for CMCL to eliminate the spurious correlation in datasets and enhance generalization. Due to the similarity to ImageNet, DeepAll, which combines all data from source domains to train a model, acts as a strong baseline and beats a series of DG methods, {\it e.g.}, DSON, MetaNorm, SagNet, MMD-AAE, CrossGrad and RSC. Nevertheless, our method still lifts the performance of DeepAll by a margin of 4.12\% on Art, and  0.99\% on average accuracy. Besides, comparable results with other competitors are also obtained on other domains.
From Table \ref{table:Office-Home50}, we can observe that CMCL exceeds all comparison methods and achieves the best results on all domains. When using a larger backbone, which replaces ResNet-18 with ResNet-50, the performance of CMCL gets significantly improved with a large margin of 5.15\%, demonstrating that our method has a non-trivial improvement in the generalization on unseen domains.

\begin{table}[t]
	\renewcommand{\arraystretch}{1.2}
	\centering
	{
		\caption{Performance comparison  on Office-Home dataset.}
		\label{table:Office-Home}
		\scalebox{0.95}{
			\begin{threeparttable}
				\begin{tabular}{l c c c c c}
					\toprule[1.5pt]
					\multicolumn{1}{l}{Dataset} & \multicolumn{5}{c}{Office-Home} \\
					\midrule
					\multicolumn{1}{l}{Backbone} & \multicolumn{5}{c}{ResNet-18} \\
					\midrule	
					
					Domain  &   Artistic  & Clipart   &  Product   &   Real World    &  Avg.  \\
					
					\midrule
					\multicolumn{6}{c} {\it{w/ Prior Knowledge about Domain Shift}} \\
					\midrule

					DDAIG \cite{digits-dg-AAAI-2020}	 & $59.20$         & $52.30$          & $\underline{74.60}$          & ${76.00}$          & $65.50$  \\
					
					DSON \cite{dson-eccv-2020}     & $59.37$         & $45.70$          & $71.84$          & $74.68$          & $62.90$          \\
					MixStyle  \cite{mixstyle-iclr-2021}    & $58.70$    & $\underline{53.40}$          & $74.20$          & $75.90$     & $65.50$           \\
					MetaNorm	\cite{metanorm-iclr-2021}	 & $59.77$         & $45.98$    & $73.13$     & $75.29$     & $63.55$           \\
					
					SagNet	\cite{sagnet-cvpr-2021} & $\underline{60.20}$         & $45.38$      & $70.42$      & $73.38$          & $62.34$    \\
					
					\midrule
					\multicolumn{6}{c} {\it{w/o Prior Knowledge about Domain Shift}} \\
					\midrule
					DeepAll \cite{digits-dg-AAAI-2020}   & $58.90$         & $49.40$          & $74.30$          & $\underline{76.20}$          & $64.70$         \\
					MMD-AAE	\cite{mmd-aae-cvpr-2018}	 & $56.50$   & $47.30$     & $72.10$          & $74.80$          & $62.70$ \\
					CrossGrad \cite{crossgrad-iclr-2018}	 & $58.40$   & $49.40$     & $73.90$          & $75.80$          & $64.40$ \\
					L2A-OT \cite{l2a-ot-eccv-2020}  \ \    & $\textbf{60.60}$   & $50.10$    & $\textbf{74.80}$    & $\textbf{77.00}$     & $\underline{65.60}$   \\
					RSC \cite{rsc-eccv-2020}			& $58.42$         & $47.90$          & $71.63$          & $74.54$          & $63.12$    \\

					\hline	
					CMCL  & $60.14$   & $\textbf{53.52}$ & $73.57$ & $75.53$ & $\textbf{65.69}$ \\
					\bottomrule[1.5pt]
				\end{tabular}
				\footnotesize
				The best result is in bold face. Underlined ones represent the second-best results. 
	\end{threeparttable}}}
\end{table}

\begin{table}[t]
	\renewcommand{\arraystretch}{1.2}
	\centering
	{
		\caption{Performance comparison  on Office-Home dataset.}
		\label{table:Office-Home50}
		\scalebox{0.95}{
			\begin{threeparttable}
				\begin{tabular}{l c c c c c}
					\toprule[1.5pt]
					\multicolumn{1}{l}{Dataset} & \multicolumn{5}{c}{Office-Home} \\
					\midrule
					\multicolumn{1}{l}{Backbone} & \multicolumn{5}{c}{ResNet-50} \\
					\midrule	
					Domain  &   Artistic  & Clipart   &  Product   &   Real World    &  Avg.  \\
					
					\midrule
					\multicolumn{6}{c} {\it{w/ Prior Knowledge of Domain Shift}} \\
					\midrule
					SagNet	\cite{sagnet-cvpr-2021}  & $63.40$   & $\underline{54.80}$  & $75.80$   & $78.30$  & $68.10  $      \\
					\midrule
					\multicolumn{6}{c} {\it{w/o Prior Knowledge of Domain Shift}} \\
					\midrule
					ERM \cite{domainbed-iclr-2021}  &$ 61.30$       & $52.40 $      & $75.80$     &$ 76.60   $ & $66.50 $                \\
					IRM  \cite{irm-2019-arxiv}                & $58.90$     & $52.20 $     &$ 72.10$    & $74.00$       &$ 64.30$                 \\
					DRO \cite{dro-arxiv-2019}   &$ 60.40  $     & $52.70$     &$75.00 $      &$ 76.00 $      & $66.00$         \\
					Mixup \cite{mixup-iclr-2018}  & $62.40 $    & $\underline{54.80} $     &$ \underline{76.90} $      & $78.30$        & $68.10 $                \\
					MLDG  \cite{mldg-aaai-2018} &$ 61.50 $      & $53.20 $      & $75.00 $      & $77.50 $      & $66.80  $               \\
					CORAL \cite{coral-iccv-2019}   &$ \underline{65.30}$     & $54.40  $     &$ 76.50$      & $\underline{78.40}$      & $\underline{68.70}  $               \\
					MMD \cite{mmd-icml-2015}   & $60.40$        &$ 53.30 $      & $74.30$      & $77.40$        & $66.30 $                \\
					DANN \cite{dann-icml-2015} & $59.90$      & $53.00$   &$73.60 $      & $76.90$      & $65.90$                \\
					CDANN \cite{cdann-eccv-2018}  &$ 61.50$     &$ 50.40$     & $74.40$       &$ 76.60$       & $65.80$               \\
					MTL \cite{mtl-jmlr-2021} &$ 61.50$      &$ 52.40$    & $74.90$      & $76.80$       & $66.40$                \\
					ARM \cite{arm-nips-2021}  & $58.90$    &$ 51.00$      & $74.10$      & $75.20$      & $64.80$                 \\
					VREx \cite{re-icml-2021} &$ 60.70$      &$ 53.00$      & $75.30$      & $76.60$       & $66.40$                \\
					RSC \cite{rsc-eccv-2020}			 &$ 60.70$     & $51.40$       &$ 74.80$      & $75.10$     &$ 65.50$                \\
					\midrule
					
					CMCL  & $\textbf{67.22}$   & $\textbf{57.88}$ & $\textbf{78.47}$ & $\textbf{79.79}$ & $\textbf{70.84}$ \\
					\bottomrule[1.5pt]
				\end{tabular}
				\footnotesize
				The best result is in bold face. Underlined ones represent the second-best results.
	\end{threeparttable}}}
\end{table}

\subsubsection{Evaluation on miniDomainNet}
We additionally carry out experiments on a large-scale dataset, miniDomainNet and report results in Table \ref{table:miniDomainNet}. For a fair comparison, we cite the results of comparison methods from a published work \cite{causalityinspired-cvpr-2022}.
We can observe that CMCL achieves the best performance on all domains and outperform the second-best method by a large margin of 3.22\% on average accuracy. Our method obtains a more significant improvement on the baseline when the dataset gets larger, which further proves the superiority of CMCL. 
\begin{table}[t]
	\renewcommand{\arraystretch}{1.2}
	\centering
	{
		\caption{Performance comparison on miniDomainNet.}
		\label{table:miniDomainNet}
		\scalebox{1.08}{
			\begin{threeparttable}
				\begin{tabular}{l c c c c c}
					\toprule[1.5pt]
					\multicolumn{1}{l}{Dataset} & \multicolumn{5}{c}{miniDomainNet} \\
					\midrule
					\multicolumn{1}{l}{Backbone} & \multicolumn{5}{c}{ResNet-18} \\
					\midrule
					Domain &   Clipart  &  Painting  &  Real   &   Sketch   &   Avg.  \\
					\midrule 
					\multicolumn{6}{c} {\it{w/ Prior Knowledge about Domain Shift}} \\
					\midrule
					
					SagNet	\cite{sagnet-cvpr-2021}  & $65.00$         & $58.10$          & $64.20$          & $58.10$          & $61.35$  \\
					\midrule
					\multicolumn{6}{c} {\it{w/o Prior Knowledge about Domain Shift}} \\
					\midrule
					DeepAll \cite{erm-tnn-1999}   & $65.50$         & $57.10$          & $62.30$          & $57.10$          & $60.50$         \\
					DRO \cite{dro-arxiv-2019}   & $64.80$         & $57.40$          & $61.50$          & $56.90$          & $60.15$  \\
					Mixup \cite{mixup-iclr-2018} & $\underline{67.10}$         & $59.10$          & $64.30$          & $59.20$          & $62.42$  \\
					MLDG  \cite{mldg-aaai-2018} & $65.70$         & $57.00$          & $63.70$          & $58.10$          & $61.12$  \\
					CORAL \cite{coral-iccv-2019}  & $66.50$         & $\underline{59.50}$          & $\underline{66.00}$          & $\underline{59.50}$          & $\underline{62.87}$  \\
					MMD \cite{mmd-icml-2015}  & $65.00$         & $58.00$          & $63.80$          & $58.40$          & $61.30$  \\
					MTL \cite{mtl-jmlr-2021} & $65.30$         & $59.00$          & $65.60$          & $58.50$          & $62.10$  \\

					\midrule
					CMCL  & $\textbf{69.34}$   & $\textbf{63.31}$ & $\textbf{68.55}$ & $\textbf{63.17}$ & $\textbf{66.09}$ \\
					\bottomrule[1.5pt]
				\end{tabular}
				\footnotesize
				The best result is in bold face. Underlined ones represent the second-best results. 
	\end{threeparttable}}}
\end{table}

\subsection{Further Analysis}
In this subsection, we conduct a series of experiments to further analyze our method.

\subsubsection{Rationality of Assumption \ref{ass:mda}}
In Subsection \ref{sec:objective formalization}, an assumption, {\it i.e.},  the ground-truth marginal distribution lies in the convex hull of source domains, is proposed as the basis of problem formalization. 
Under this assumption, a test domain sampled from the ground-truth distribution with a selection bias is naturally near to the ground-truth distribution and lies inside of the convex hull.
Here we empirically analyze the rationality of the assumption. As shown at the second row in Figure~\ref{fig:datasets}~(a), MNIST-M is obviously different from other domains, the domain shift in it is obviously different from that of others. MNIST-M probably does not lie inside of the convex hull of other domains, which means that the assumption is not well met. From Table \ref{table:Digits-DG}, we can observe that all reported domain generalization methods perform worst in MNIST-M among all test domains. Hence we can conclude that Assumption \ref{ass:mda} is necessary and reasonable for distribution alignment based domain generalization.

\subsubsection{Effectiveness of Each Component of CMCL}
We discuss the effectiveness of $\mathcal{L}_\text{mean}$ and $\mathcal{L}_\text{cov}$ in Eq.(\ref{mm}), Mean Classifier in Eq.(\ref{cls_mean}) and EMA in Eq.(\ref{mu}). The results on PACS with ResNet-18 as the feature extractor are reported in Table \ref{table:components}. As shown in Table \ref{table:components}, we can observe that removing any component of CMCL can lead to significant performance degradation, demonstrating the effectiveness of our design. 

\begin{table}[t]
	\renewcommand{\arraystretch}{1.2}
	\centering
	{
		\caption{Impact of different components on performance of our CMCL.}
		\label{table:components}
		\scalebox{0.9}{
			\begin{threeparttable}
				\begin{tabular}{l c c c c c}
					\toprule[1.5pt]
					\multicolumn{1}{l}{Dataset} & \multicolumn{5}{c}{PACS} \\
					\midrule
					\multicolumn{1}{l}{Backbone} & \multicolumn{5}{c}{ResNet-18} \\
					\midrule
					\multicolumn{1}{l}{Domain}  &   Art  & Cartoon   &   Photo   &  Sketch    &   Avg.   \\
					\midrule
					DeepAll\cite{dger-nips-2020}    & $78.93$  & $75.02$ & $\textbf{96.60}$  & $70.48$ & $80.25$        \\
					CMCL w/o  $\mathcal{L}_\text{mean}$ \& $\mathcal{L}_\text{cov}$   & $83.02$         & $77.63$          & $94.63$          & $81.45$          & $84.18$         \\
					CMCL w/o  $\mathcal{L}_\text{mean}$   & $83.51$         & $78.07$          & $94.87$          & $80.73$          & $84.29$         \\
					CMCL w/o  $\mathcal{L}_\text{cov}$   & $\underline{83.74}$         & $\underline{80.05}$          & $94.87$          & $81.52$          & $\underline{85.05}$         \\
					CMCL w/o  Mean Classifier   & $84.03$         & $79.11$          & ${94.33}$          & $81.78$          & $84.81$         \\
					CMCL w/o  EMA  & $83.33$         & $79.12$          & $93.83$          & $\underline{82.80}$          & $84.77$         \\
					CMCL    &$\textbf{84.55}$   & $\textbf{80.08}$ & $\underline{94.95}$ & $\textbf{82.86}$ & $\textbf{85.61}$  \\
					\bottomrule[1.5pt]
				\end{tabular}
				\footnotesize
				The best result is in bold face. Underlined ones represent the second-best results. 
	\end{threeparttable}}}
\end{table}


\begin{figure}[!htbp]
	\setlength{\abovecaptionskip}{0pt}
	\setlength{\belowcaptionskip}{0pt}
	\renewcommand{\figurename}{Figure}
	\centering
        \subfigure[$\lambda_2 = 100$]
	{\includegraphics[width=0.225\textwidth]{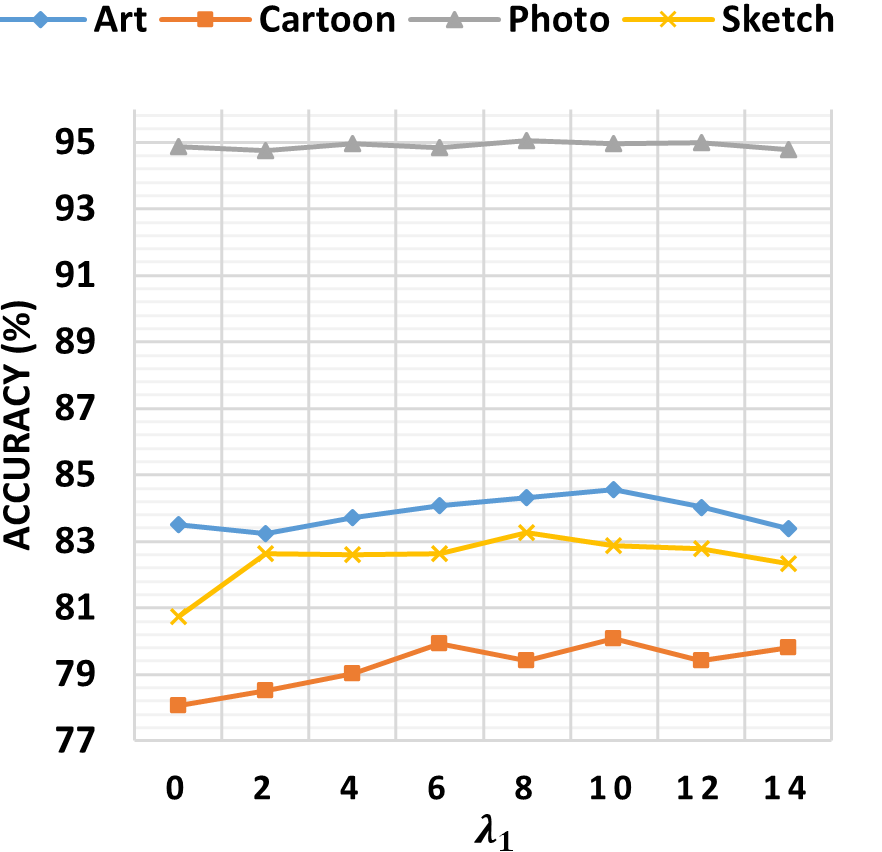}}
         \subfigure[$\lambda_1 = 10$]
	{\includegraphics[width=0.225\textwidth]{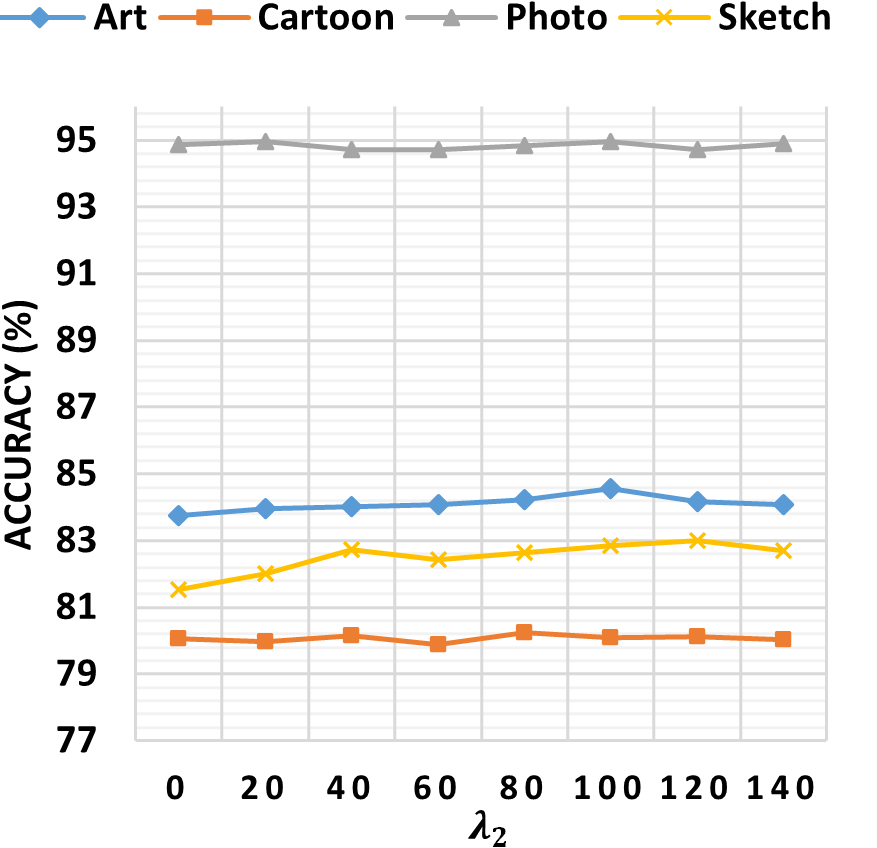}}
 	\caption{{Parameter sensitivity analysis with respect to $\lambda_1$ and $\lambda_2$  on PACS datasets.}}
	\label{fig:parameter}
\end{figure}

\begin{figure}[!htbp]
	\setlength{\abovecaptionskip}{0pt}
	\setlength{\belowcaptionskip}{0pt}
	\renewcommand{\figurename}{Figure}
	\centering
        \subfigure[Art]
	{\includegraphics[width=0.225\textwidth]{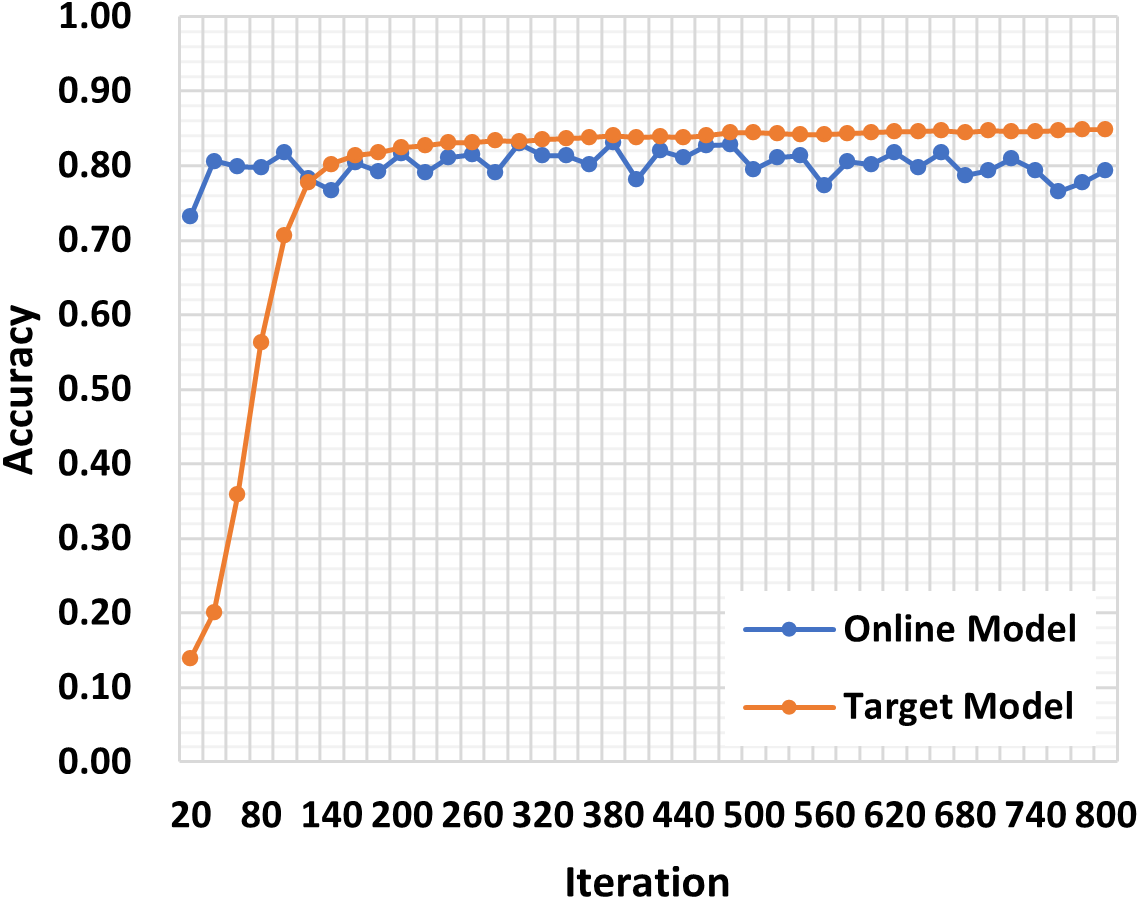}}
         \subfigure[Cartoon]
	{\includegraphics[width=0.225\textwidth]{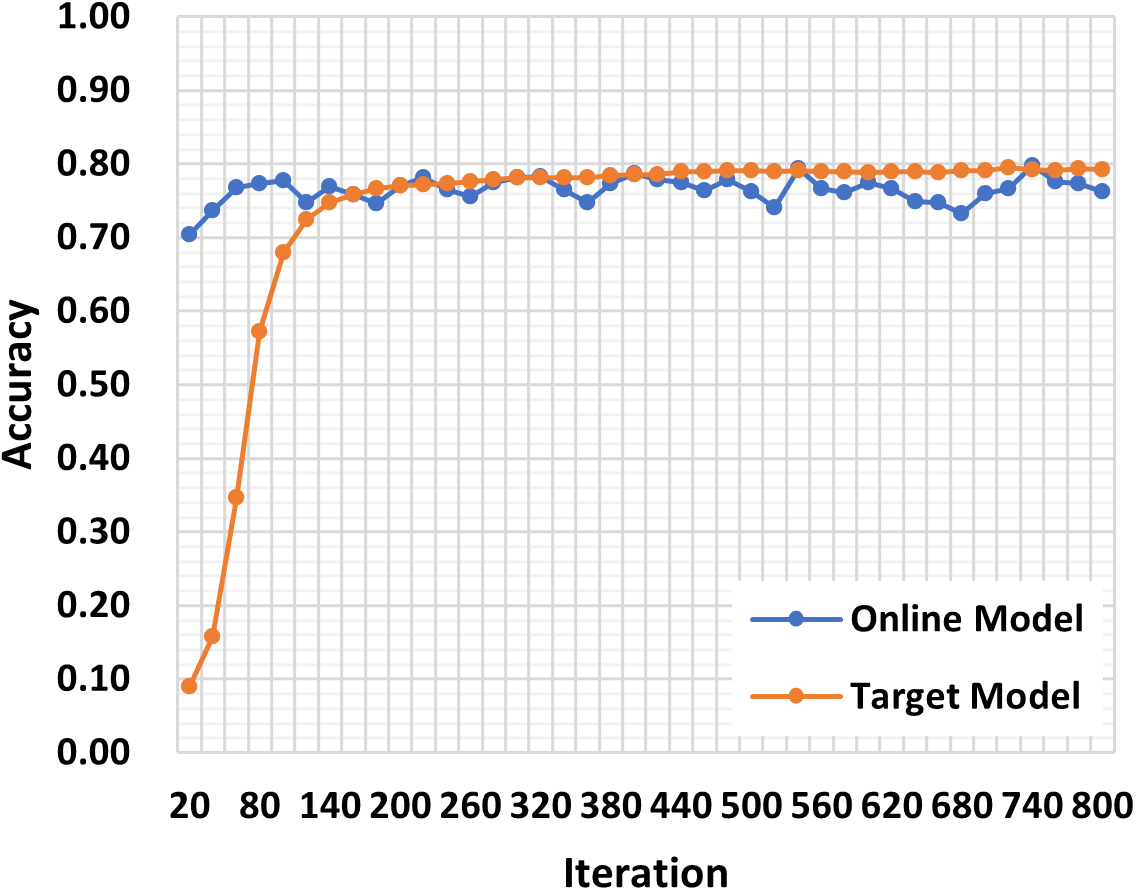}}

        \subfigure[Photo]
	{\includegraphics[width=0.225\textwidth]{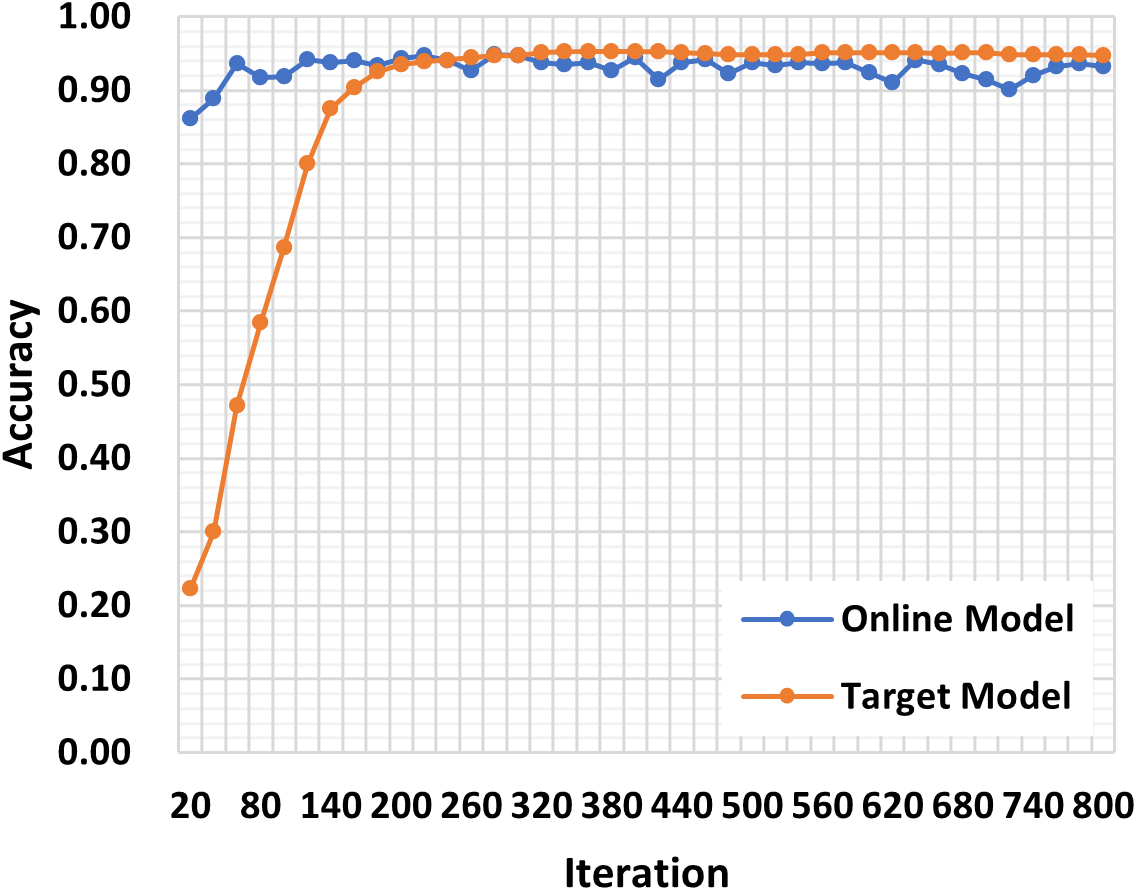}}
         \subfigure[Sketch]
	{\includegraphics[width=0.225\textwidth]{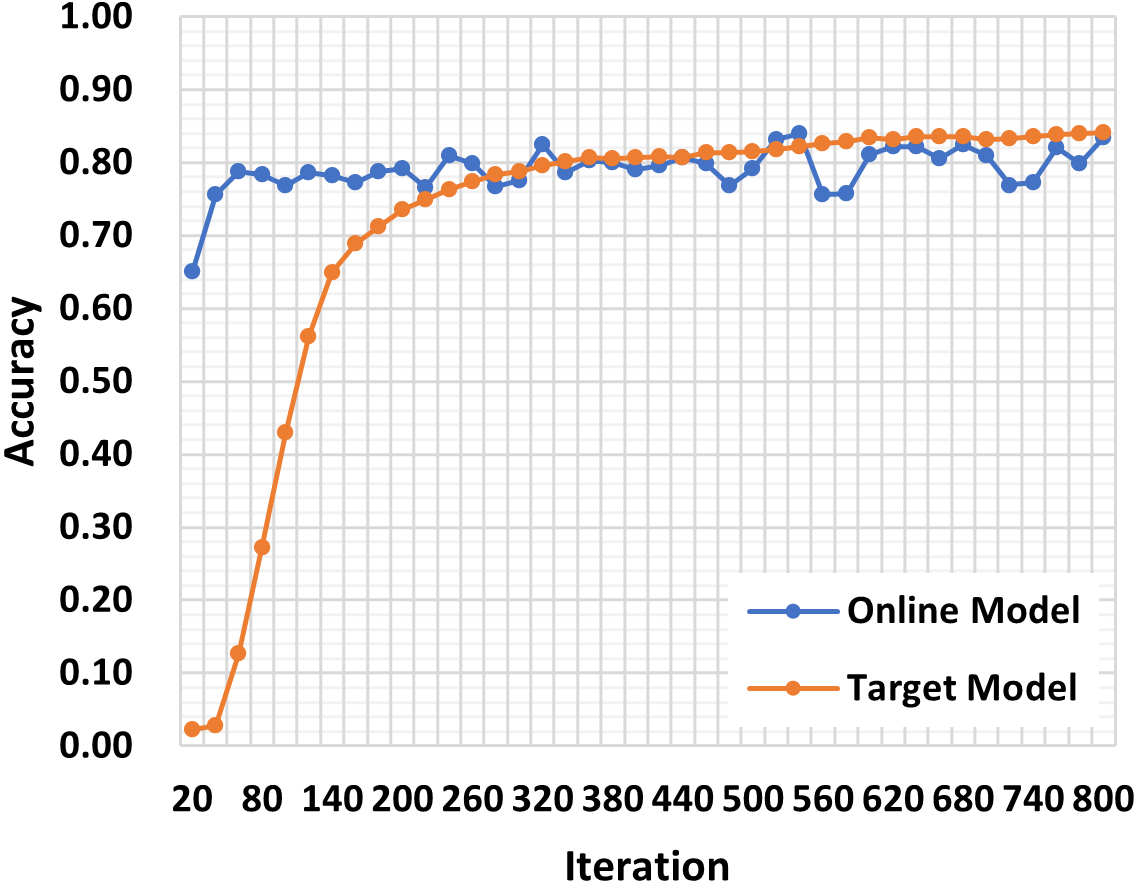}}
	\caption{{Training accuracy curves of online model and target model on PACS datasets.}}
 
	\label{fig:curve}
\end{figure}


CMCL w/o  $\mathcal{L}_\text{mean}$ \& $\mathcal{L}_\text{cov}$ is a variant which optimizes the objective in Eq.(\ref{final optimization objective}) without considering the constrained condition of marginal distribution alignment. 
Unless the marginal distributions of source domains are naturally aligned, the optimization objective of CMCL w/o  $\mathcal{L}_\text{mean}$ \& $\mathcal{L}_\text{cov}$ is obviously different from CMCL. The former only considers minimizing the KL-divergence between domain-specific posterior distributions given samples from source domains, but the latter tries to minimize the KL-divergence given any samples from the real world, which is more general. We can see that the latter works better on PACS. 

As shown in Eq.(\ref{mm}), we adopt $\lambda_1$ and $\lambda_2$ to control the penalty intensity of $\mathcal{L}_\text{mean}$ and $\mathcal{L}_\text{cov}$ respectively. To illustrate the significance of $\lambda_1$ and $\lambda_2$, parameter sensitivity analysis is also conducted  as shown in Figure \ref{fig:parameter}.
We update the global classifier $W^g$ by the mean of domain-specific classifiers, denoted as Mean Classifier, at Stage B. As the training progresses, the domain-specific classifiers tend to be consistent. So the Mean Classifier is a reasonable prediction of the convergence of domain-specific classifiers. From Table \ref{table:components}, we can see that Mean Classifier makes a significant contribution to the final performance. 
As demonstrated in Subsection \ref{sec:AO}, EMA helps to improve the generalization of the trained model and the stability of the training process. The former can be verified by CMCL w/o EMA in Table \ref{table:components}. The latter is further illustrated in Figure \ref{fig:curve}. The training accuracy curves of target models which are updated by ERM are smoother than that of online models.

\subsubsection{Effectiveness of Maximum In-Domain Likelihood}
\begin{table}[t]
	\renewcommand{\arraystretch}{1.2}
	\centering
	{
		\caption{Performance comparison among different variants on PACS datasets.}
		\label{table:variants}
		\scalebox{1}{
			\begin{threeparttable}
				\begin{tabular}{l c c c c c}
					\toprule[1.5pt]
					\multicolumn{1}{l}{Dataset} & \multicolumn{5}{c}{PACS} \\
					\midrule
					\multicolumn{1}{l}{Backbone} & \multicolumn{5}{c}{ResNet-18} \\
					\midrule
					Domain  &   Art  & Cartoon   &   Photo   &  Sketch    &   Avg.   \\
					\midrule
					CMCL-KL \quad \quad \quad \quad \ \   & $83.36$         & $\underline{79.32}$          & $94.49$          & $80.76$          & $84.48$         \\
					E2E-KL   & $\underline{83.53}$         & $79.31$          & $\textbf{95.01}$          & $\underline{80.77}$          & $\underline{84.66}$         \\
					CMCL    &$\textbf{84.55}$   & ${\textbf{80.08}}$ & $\underline{94.95}$ & $\textbf{82.86}$ & $\textbf{85.61}$  \\
					\bottomrule[1.5pt]
				\end{tabular}
				\footnotesize
				The best result is in bold face. Underlined ones represent the second-best results. 
	\end{threeparttable}}}
\end{table}

The term of maximum in-domain likelihood in Eq.(\ref{fml:original-objective}) is proposed to eliminate the side-effect of KL-divergence as mentioned in Section \ref{sec:Maximum In-Domain Likelihood}, and then the constrained maximum cross-domain likelihood optimization problem can be deduced.
To evaluate the effectiveness of the term of maximum in-domain likelihood, CMCL-KL, a variant of CMCL, is constructed by removing the term of maximum in-domain likelihood in Eq.(\ref{fml:original-objective}). CMCL-KL is also optimized by an alternating optimization strategy. Specifically, the first term of Eq.(\ref{cdl}) is removed and the KL-divergence between different domain-specific posterior distributions, which are parameterized by frozen softmax layers at Stage C, is directly minimized. For a fair comparison, the other components of CMCL keep unchanged.
From Table \ref{table:variants}, we can observe that CMCL-KL performs worse than CMCL even though CMCL-KL is intuitively plausible. 
The optimization objective of maximum cross-domain likelihood of CMCL, which is strictly deduced from the original general objective with the term of maximum in-domain likelihood instead of intuitive motivation, has a solid theoretical basis and works well as expected.

\subsubsection{Effectiveness of Alternating Optimization Strategy}
In order to demonstrate the effectiveness of the proposed alternating optimization strategy, E2E-KL is proposed, which directly optimizes the original objective in Eq.(\ref{fml:original-objective}) in an end-to-end manner. For a fair comparison, the other components of CMCL keep unchanged, including marginal distribution alignment.
From Table \ref{table:variants}, we can observe that CMCL works better than E2E-KL. Although the optimization objective of CMCL is deduced from that of E2E-KL, an appropriate optimization strategy has a significant impact on the final performance due to the non-convexity of the neural networks.

\subsubsection{Feature Visualization}
To qualitatively assess the ability of CMCL in learning the domain-invariant classifier, we visualize the distribution of the learned features using t-SNE \cite{t-sne} in Figure \ref{fig:visual}. 
Comparing the feature distribution of DeepAll and our CMCL, we can observe that DeepAll has a satisfactory capacity for learning discriminative features, but fails to align distributions of different domains. Source domain features and test domain features derived by CMCL are aligned better than that of DeepAll, and the learned features are separated better according to their semantic categories. 
\begin{figure}[t]
	\setlength{\abovecaptionskip}{0pt}
	\setlength{\belowcaptionskip}{0pt}
	\renewcommand{\figurename}{Figure}
	\centering
	\includegraphics[width=0.49 \textwidth]{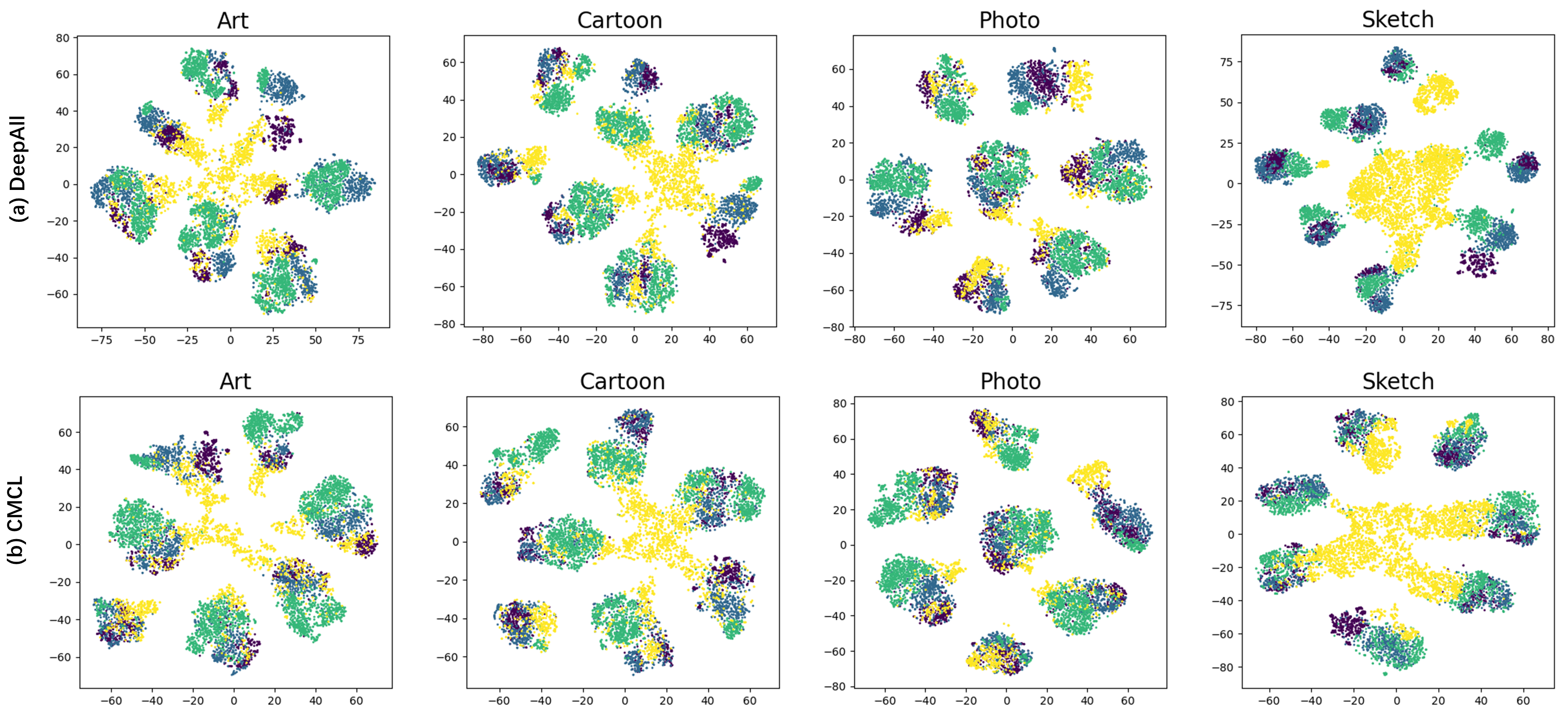}
	\caption{{t-SNE visualization on PACS with ResNet-18 as the feature extractor. The yellow points denote features of the test domain and others are that of source domains. (a) Features extracted by the baseline approach of DeepAll. (b) Features extracted by proposed CMCL. }}
	\label{fig:visual}
\end{figure}
We can see that though the class-conditional distributions are not exactly matched, CMCL still achieves an excellent performance due to the learned domain-invariant representation space where optimal domain-specific classifiers are nearly the same, which is the original motivation of CMCL. The visualization further proves the effectiveness of CMCL in learning domain generalizable features and domain-invariant classifiers.

\section{Conclusions} \label{section:conclusion}
In this paper, a novel domain generalization method named CMCL is proposed to learn generalizable and discriminative representations via constrained maximum cross-domain likelihood.
We firstly formalize an optimization problem in which an expectation of KL-divergence between posterior distributions is minimized. In this original formalization, the expectation is computed on the ground-truth marginal distribution for better generalizability.
We propose a term of maximum in-domain likelihood to eliminate the side-effect of KL-divergence, {\it i.e.}, entropy increase. Furthermore, a constraint of marginal distribution alignment is proposed to approximate the ground-truth marginal distribution with source domains under a convex hull assumption. Finally, a more concise and effective constrained maximum cross-domain likelihood optimization problem is deduced.
The constrained optimization problem is transformed into an unconstrained one by adding a penalty term and approximately solved by an alternating optimization strategy. CMCL naturally realizes the joint distribution alignment by solving this optimization problem.
Comprehensive experiments on four datasets demonstrate that our method can obtain excellent domain generalization performance.

In this work, we propose an important convex hull assumption, under which the domain-invariant classifier could generalize to unseen domains. 
In the future, generative methods can be coupled with CMCL to diversify training domains so that the marginal distribution of the real world is more likely to be located in the convex hull of that of training domains.

\section*{Acknowledgment}
The authors are thankful for the financial support by the Key-Area Research and Development Program of Guangdong Province 2019B010153002, the National Natural Science Foundation of China (62106266, U1936206).


\begin{thebibliography}{99}	
\bibitem{UDA-Long-2016} M. Long, H. Zhu, J. Wang, S. Wang, and M.I. Jordan, ``Unsupervised domain adaptation with residual transfer networks," in {\it Proc. Advances Neural Inf. Process. Syst.}, 2016.

\bibitem{DGMI-Li-2020} H. Li, Y. Wang, R. Wan, S. Wang, T. Li, and A. C. Kot, ``Domain generalization for medical imaging classification with linear-dependency regularization," in {\it Proc. Advances Neural Inf. Process. Syst.}, pp. 3118-3129, 2020.

\bibitem{dsa-tmi-2020} L. Zhang X. Wang, D. Yang, T. Sanford, S. Harmon, B. Turkbey, B. J. Wood, H. Roth, A. Myronenko, D. Xu, and Z. Xu, ``Generalizing deep learning for medical image segmentation to unseen domains via deep stacked transformation," {\it IEEE Trans. Med. Imag.}, vol. 39, no. 7, pp. 2531-2540, July. 2020.

\bibitem{ifr-wang-2013} K. Muandet, D. Balduzzi, and B. Scholkopf, ``Domain generalization via invariant feature representation," in {\it  Proc. Int. Conf. Mach. Learn.}, pp. 10-18, 2013.

%

\bibitem{mmdaae-cvpr-2018} H. Li, S. J. Pan, S. Wang, and A. C. Kot, ``Domain generalization with adversarial feature learning," in {\it Proc. IEEE Conf. Comput. Vis. Pattern Recognit.}, pp. 5400-5409, 2018.

\bibitem{Li_TDC-cian-eccv-2018} Y. Li,  X. Tian, M. Gong, Y. Liu, T. Liu, K. Zhang and D. Tao, ``Deep domain generalization via conditional invariant adversarial networks," in {\it Proc. Eur. Conf. Comput. Vis.}, pp. 624-639, 2018.

\bibitem{Li-TDC-cir-AAAI-2018} Y. Li,  M. Gong, X. Tian, T. Liu, and D. Tao, ``Domain generalization via conditional invariant representations," in \emph{Proc. AAAI Conf. Artif. Intell.}, 2018.

%

\bibitem{bound-2021-arxiv} I. Albuquerque, J. Monteiro, M. Darvishi, T. H. Falk, and I. Mitliagkas, ``Generalizing to unseen domains via distribution matching," {\it arXiv preprint arXiv:1911.00804v6}, 2021.

\bibitem{mbdg-nips-2021} A. Robey, G. J. Pappas, and H. Hassani, ``Model-based domain generalization," in {\it Proc. Advances Neural Inf. Process. Syst.}, pp. 20210-20229, 2021.

\bibitem{re-icml-2021} D. Krueger, E. Caballero, J.-H. Jacobsen, A. Zhang, J. Binas, D. Zhang, R. L. Priol, and A. Courville, ``Out-of-distribution generalization via risk extrapolation (rex)," in {\it  Proc. Int. Conf. Mach. Learn.}, pp. 5815-5826, 2021.

\bibitem{cdann-eccv-2018} Y. Li, X. Tian, M. Gong, Y. Liu, T. Liu, K. Zhang, and D. Tao, ``Deep domain generalization via conditional invariant adversarial networks," in {\it Proc. Eur. Conf. Comput. Vis.}, pp. 624-639, 2018.

%

\bibitem{icp-jci-2018} C. Heinze-Deml, J. Peters, and N. Meinshausen, ``Invariant causal prediction for nonlinear models," {\it J. Causal Inference}, vol. 6, no. 2, pp. 20170016, 2018.

\bibitem{irm-2019-arxiv} M. Arjovsky, L. Bottou, I. Gulrajani, and D. Lopez-Paz, ``Invariant risk minimization," {\it arXiv preprint arXiv:1907.02893}, 2019.

\bibitem{mip-2020-arxiv} M. Koyama, S. Yamaguchi, ``Out-of-distribution generalization with maximal invariant predictor," {\it arXiv preprint arXiv:1907.02893}, 2020.

\bibitem{ip-nips-2021} K. Ahuja, E. Caballero, D. Zhang, J.-C. Gagnon-Audet, Y. Bengio, I. Mitliagkas, and I. Rish, ``Invariance principle meets information bottleneck for out-of-distribution generalization," in {\it Proc. Advances Neural Inf. Process. Syst.}, pp. 3438-3450, 2021.

\bibitem{iib-AAAI-2022} B. Li, Y. Shen, Y. Wang, W. Zhu, C. Reed, K. Keutzer, D. Li, and H. Zhao, ``Invariant information bottleneck for domain generalization," in \emph{Proc. AAAI Conf. Artif. Intell.}, 2022.

\bibitem{olirda-icml-2019} H. Zhao, R. T. d. Combes, K. Zhang, and G. J. Gordon, ``On learning invariant representations for domain adaptation," in {\it  Proc. Int. Conf. Mach. Learn.}, pp. 7523-7532, 2019.

\bibitem{kl-ams-1951} S. Kullback, R. A. Leiblerg, ``On information and sufficiency," \emph{Ann. Math. Stat.}, vol.22, no.1, pp. 79-86, 1951.






\bibitem{progressive-tnnls-2019} J. Li, K. Lu, Z. Huang, L. Zhu and H. T. Shen, ``Heterogeneous domain adaptation through progressive alignment," {\it IEEE Trans. Neural Netw. Learn. Syst.}, vol. 30, no. 5, pp. 1381-1391, May 2019.

\bibitem{nem-tnnls-2020} Z. Wang, B. Du and Y. Guo, ``Domain adaptation with neural embedding matching," {\it IEEE Trans. Neural Netw. Learn. Syst.}, vol. 31, no. 7, pp. 2387-2397, July 2020.

\bibitem{graph-tcsvt-2022} J. Sun, Z. Wang, W. Wang, H. Li, F. Sun and Z. Ding, ``Joint adaptive dual graph and feature selection for domain adaptation," {\it IEEE Trans. Circuits Syst. Video Technol.}, vol. 32, no. 3, pp. 1453-1466, Mar. 2022.

\bibitem{mmd} A. Gretton, K. M. Borgwardt, M. J. Rasch, B. Schölkopf, and A. Smola, ``A kernel two-sample test," \emph{J. Mach. Learn. Res.}, vol. 13, pp. 723–773,
Mar. 2012.

\bibitem{cmd-iclr-2017} W. Zellinger, T. Grubinger, E. Lughofer, T. Natschläger, and S. Saminger-Platz, ``Central moment discrepancy (cmd) for domain invariant representation learning," in {\it Proc. Int. Conf. Learn. Representat.}, 2017.

\bibitem{wd-aaai-2020} S. Zhao, G. Wang, S. Zhang, Y. Gu, Y. Li, Z. Song, P. Xu, R. Hu, H. Chai, and K. Keutzer, ``Multi-source distilling domain adaptation," in \emph{Proc. AAAI Conf. Artif. Intell.}, 2020.

\bibitem{deepmmd-2014-arxiv} E. Tzeng, J. Hoffman, N. Zhang, K. Saenko, and T. Darrell, ``Deep domain confusion: maximizing for domain invariance," {\it arXiv preprint arXiv:1412.3474}, 2014.

\bibitem{dcoral-eccv-2016} B. Sun,  K. Saenko, ``Deep coral: Correlation alignment for deep domain adaptation," in {\it Proc. Eur. Conf. Comput. Vis.}, pp. 443-450, 2016.

\bibitem{djdn-icml-2017} M. Long, H. Zhu, J. Wang, and M. I. Jordan, ``Deep transfer learning with joint adaptation networks," in {\it  Proc. Int. Conf. Mach. Learn.}, pp. 2208-2217, 2017.

\bibitem{dann-icml-2015} Y. Ganin, V. Lempitsky, ``Unsupervised domain adaptation by backpropagation," in {\it  Proc. Int. Conf. Mach. Learn.}, pp. 1180–1189, 2015.

\bibitem{mada-aaai-2018} Z. Pei, G. Cao, M. Long, and J. Wang, ``Multi-adversarial domain adaptation," in \emph{Proc. AAAI Conf. Artif. Intell.}, 2018.

\bibitem{ajada-tcsvt-2022} Y. Zuo, H. Yao, L. Zhuang and C. Xu,  ``Margin-based adversarial joint alignment domain adaptation," {\it IEEE Trans. Circuits Syst. Video Technol.}, vol. 32, no. 4, pp. 2057-2067, Apr. 2022.

\bibitem{wdg-tnnls-2017} L. Niu, W. Li, D. Xu and J. Cai, ``Visual recognition by learning from web data via weakly supervised domain generalization," {\it IEEE Trans. Neural Netw. Learn. Syst.}, vol. 28, no. 9, pp. 1985-1999, Sept. 2017.

\bibitem{emv-tnnls-2018} L. Niu, W. Li, D. Xu and J. Cai, ``An exemplar-based multi-view domain generalization framework for visual recognition," {\it IEEE Trans. Neural Netw. Learn. Syst.}, vol. 29, no. 2, pp. 259-272, Feb. 2018.

\bibitem{metareg-nips-2018} Y. Balaji, S. Sankaranarayanan, and R. Chellappa, ``MetaReg: towards domain generalization using meta-regularization," in {\it Proc. Advances Neural Inf. Process. Syst.}, pp. 1006-1016, 2018.

\bibitem{masf-nips-2019} Q. Dou, D. C. Castro, K. Kamnitsas, and B. Glocker, ``Domain generalization via model-agnostic learning of semantic features," in {\it Proc. Advances Neural Inf. Process. Syst.}, pp. 6450-6461, 2019.

\bibitem{dro-2016-arxiv} J. C. Duchi, P. W. Glynn, and H. Namkoong, ``Statistics of robust optimization: A generalized empirical likelihood approach," {\it Math. Oper. Res.}, vol. 46, no. 3, pp. 946-969, 2021.

\bibitem{dro-arxiv-2019} S. Sagawa, P. W. Koh, T. B. Hashimoto, and P. Liang, ``Distributionally robust neural networks for group shifts: On the importance of regularization for worst-case generalization,"  in {\it Proc. Int. Conf. Learn. Representat.}, 2020.

\bibitem{mixstyle-iclr-2021} K. Zhou, Y. Yang, Y. Qiao, and X. Tao, ``Domain generalization with mixstyle," in {\it Proc. Int. Conf. Learn. Representat.}, 2021.

\bibitem{digits-dg-AAAI-2020} K. Zhou, Y. Yang, Y. M. Hospedales, and T. Xiang, ``Deep domain-adversarial image generation for domain generalisation," in \emph{Proc. AAAI Conf. Artif. Intell.}, 2020.

\bibitem{dson-eccv-2020} S. Seo,  Y. Suh, D. Kim, G. Kim, J. Han, and B. Han, ``Learning to optimize domain specific normalization for domain generalization," in {\it Proc. Eur. Conf. Comput. Vis.}, pp. 68-83, 2020.

\bibitem{style-tcsvt-2022} Y. Wang, L. Qi, Y. Shi, and Y. Gao, ``Feature-based style randomization for domain generalization," {\it IEEE Trans. Circuits Syst. Video Technol.}, early access, doi=10.1109/TCSVT.2022.3152615.

\bibitem{mmd-aae-cvpr-2018} H. Li, S. J. Pan, S. Wang, and A. C. Kot, ``Domain generalization with adversarial feature learning," in {\it Proc. IEEE Conf. Comput. Vis. Pattern Recognit.}, pp. 5400-5409, 2018.

\bibitem{ccsa-iccv-2017} S. Motiian, M. Piccirilli, D. A. Adjeroh, and G. Doretto, ``Unified deep supervised domain adaptation and generalization," in \emph{Proc. IEEE Int. Conf. Comput. Vis.}, pp. 5715-5725, 2017.

\bibitem{ddian-2021-arxiv} M. M. Rahman, C. Fookes, and S. Sridharan, ``Discriminative domain-invariant adversarial network for deep domain generalization," {\it arXiv preprint arXiv:2108.08995}, 2021.

\bibitem{dger-nips-2020} S. Zhao, VM. Gong, T. Liu, H. Fu, and D. Tao, ``Domain generalization via entropy regularization," in {\it Proc. Advances Neural Inf. Process. Syst.}, pp. 16096-16107, 2020.

\bibitem{iedg-2022-aistats} E. Rosenfeld, P. Ravikumar, and A. Risteski, ``An online learning approach to interpolation and extrapolation in domain generarization," {\it Proc. Int. Conf. Artif. Intell. Stat.}, pp. 2641-2657, 2022.




\bibitem{mnist-1998} Y. LeCun,  L. Bottou, and P. Haffner, ``Gradient-based learning applied to document recognition," {\it Proc. of the IEEE}, vol.86, no.11, pp. 2278-2324, Nov. 1998.

\bibitem{mnist-m-2015} Y. Ganin, and V. Lempitsky, ``Unsupervised domain adaptation by backpropagatio," in {\it  Proc. Int. Conf. Mach. Learn.}, pp. 1180-1189, 2015.

\bibitem{svhn-2011} Y. Netzer, T. Wang, A. Coates, A. Bissacco, B. Wu, and A. Y. Ng  ``Reading digits in natural images with unsupervised feature learning," 2011.

\bibitem{pacs-iccv-2017} D. Li, Y. Yang, Y. Song, and T. M. Hospedales, ``Deeper, broader and artier domain generalization," in \emph{Proc. IEEE Int. Conf. Comput. Vis.}, pp. 5542-5550, 2017.

\bibitem{office-home-cvpr-2017} H. Venkateswara, J. Eusebio, S. Chakraborty, and S. Panchanathan, ``Deep hashing network for unsupervised domain adaptation," in {\it Proc. IEEE Conf. Comput. Vis. Pattern Recognit.}, pp. 5018-5027, 2017.

\bibitem{jgen-cvpr-2019} F. M Carlucci, A. D'Innocent, S. Bucci, B. Caputo and T. Tommasi, ``Domain generalization by solving jigsaw puzzle," in {\it Proc. IEEE Conf. Comput. Vis. Pattern Recognit.}, pp. 2229-2238, 2019.

%

\bibitem{minidomainnet-tip-2021} K. Zhou, Y. Yang, Y. Qiao, and T. Xiang, ``Domain adaptive ensemble learning," {\it IEEE Trans. Image Process.}, vol. 30, pp. 8008-8018, Sept. 2021.

\bibitem{coral-iccv-2019} X. Peng, Q. Bai, X. Xia, Z. Huang, K. Saenko, and B. Wang, ``Moment matching for multi-source domain adaptation," in \emph{Proc. IEEE Int. Conf. Comput. Vis.}, pp. 1406-1415, 2019.

\bibitem{domainbed-iclr-2021} I. Gulrajani, D. Lopez-Paz, ``In search of lost domain generalization," in {\it Proc. Int. Conf. Learn. Representat.}, 2021.

\bibitem{stablnet-cvpr-2021} X. Zhang, P. Cui, R. Xu, L. Zhou, Y. He, and Z. Shen, ``Deep stable learning for out-of-distribution generalization," in {\it Proc. IEEE Conf. Comput. Vis. Pattern Recognit.}, pp. 5372-5382, 2021.

\bibitem{mgfa-bmvc-2021} M. H. Khan,  T. Zaidi, S. Khan, and F. S. Khan, ``Mode-guided feature augmentation for domain generalization," in {\it Proc. Brit. Mach. Vis. Conf.}, 2021.

\bibitem{crossgrad-iclr-2018} S. Shankar, V. Piratla, S. Chakrabarti, S. Chaudhuri, P. Jyothi, and S. Sarawagi, ``Generalizing across domains via cross-gradient training," in {\it Proc. Int. Conf. Learn. Representat.}, 2018.

\bibitem{l2a-ot-eccv-2020} K. Zhou, Y. Yang, T. Hospedales, and T. Xiang, ``Learning to generate novel domains for domain generalization," in {\it Proc. Eur. Conf. Comput. Vis.}, pp. 561-578, 2020.

%

\bibitem{copa-iccv-2021} G. Wu, S. Gong, ``Collaborative optimization and aggregation for decentralized domain generalization and adaptation," in \emph{Proc. IEEE Int. Conf. Comput. Vis.}, pp. 6484-6493, 2021.

\bibitem{sfa-a-iccv-2021} P. Li, D. Li, W. Li, S. Gong, Y. Fu, and T. M. Hospedales, ``A simple feature augmentation for domain generalization," in \emph{Proc. IEEE Int. Conf. Comput. Vis.}, pp. 8866-8875, 2021.

\bibitem{metanorm-iclr-2021} Y. Du, X. Zhen, L. Shao, and C. G. M. Snoek, ``Metanorm: Learning to normalize few-shot batches across domains," in {\it Proc. Int. Conf. Learn. Representat.}, 2021.

\bibitem{ldsdg-cvpr-2021} Z. Wang, Y. Luo, R. Qiu, Z. Huang, and M. Baktashmotlagh, ``Learning to diversify for single domain generalization," in {\it Proc. IEEE Conf. Comput. Vis. Pattern Recognit.}, pp. 834-843, 2021.

\bibitem{sagnet-cvpr-2021} H. Nam, H. Lee, J. Park, W. Yoon, and D. Yoo, ``Reducing domain gap by reducing style bias," in {\it Proc. IEEE Conf. Comput. Vis. Pattern Recognit.}, pp. 8690-8699, 2021.

%

\bibitem{fact-cvpr-2021} Q. Xu, R. Zhang, Y. Zhang, Y. Wang, and Q. Tian, ``A fourier-based framework for domain generalization," in {\it Proc. IEEE Conf. Comput. Vis. Pattern Recognit.}, pp. 14383-14392, 2021.

\bibitem{padin-cvpr-2021} O. Nuriel, S. Benaim, and L. Wolf, ``Permuted adain: Reducing the bias towards global statistics in image classification," in {\it Proc. IEEE Conf. Comput. Vis. Pattern Recognit.}, pp. 9482-9491, 2021.

\bibitem{efd-cvpr-2022} Y. Zhang, M. Li, R. Li, K. Jia, and L. Zhang, ``Exact feature distribution matching for arbitrary style transfer and domain generalization," in {\it Proc. IEEE Conf. Comput. Vis. Pattern Recognit.}, pp. 8035-8045, 2022.

\bibitem{dmg-eccv-2020} P. Chattopadhyay,  Y. Balaji, and J. Hoffman1, ``Learning to balance specificity and invariance for in and out of domain generalization," in {\it Proc. Eur. Conf. Comput. Vis.}, pp. 301-318, 2020.

\bibitem{eisnet-eccv-2020} S. Wang,  L. Yu, C. Li, C.-W. Fu, and P.-A. Heng, ``Learning from extrinsic and intrinsic supervisions for domain generalization," in {\it Proc. Eur. Conf. Comput. Vis.}, pp. 159-176, 2020.

\bibitem{mmld-aaai-2020} T. Matsuura, T. Harada, ``Domain generalization using a mixture of multiple latent domains," in \emph{Proc. AAAI Conf. Artif. Intell.}, 2020.

\bibitem{rsc-eccv-2020} Z. Huang,  H. Wang, E. P. Xing and D. Huang, ``Self-challenging improves cross-domain generalization," in {\it Proc. Eur. Conf. Comput. Vis.}, pp. 124-140, 2020.

\bibitem{nas-ood-iccv-2021} H. Bai, F. Zhou, L. Hong, N. Ye, S.-H. G. Chan, and Z. Li, ``Nas-ood: Neural architecture search for out-of-distribution generalization," in \emph{Proc. IEEE Int. Conf. Comput. Vis.}, pp. 8320-8329, 2021.

\bibitem{daml-cvpr-2021} S. Yang, Z. Cao, C. Wang, J. Wang, and M. Long, ``Open domain generalization with domain-augmented meta-learning," in {\it Proc. IEEE Conf. Comput. Vis. Pattern Recognit.}, pp. 9624-9633, 2021.

\bibitem{mdghybrid-icml-2021} D. Mahajan, S. Tople, and A. Sharma, ``Domain generalization using causal matching," in {\it  Proc. Int. Conf. Mach. Learn.}, pp. 7313-7324, 2021.

\bibitem{dsfg-wacv-2022} S. Yuan, Y. Li, D. Wang, K. Bai, L. Carin, and D. Carlson, ``Learning to weight filter groups for robust classification," in {\it  Proc. - IEEE/CVF Winter Conf. Appl. Comput. Vis.}, pp. 3041-3050, 2022.

\bibitem{randconv-iclr-2020} Z. Xu, D. Liu, J. Yang, C. Raffel, and M. Niethammer, ``Robust and generalizable visual representation learning via random convolutions," in {\it Proc. Int. Conf. Learn. Representat.}, 2020.

\bibitem{mixup-iclr-2018} H. Zhang, M. Cissé, Y. N. Dauphin, and D. Lopez-Paz, ``Mixup: Beyond empirical risk minimization," in {\it Proc. Int. Conf. Learn. Representat.}, 2018.

\bibitem{mldg-aaai-2018} D. Li, Y. Yang, Y.-Z. Song, and T. M. Hospedales., ``Learning to generalize: Meta-learning for domain generalization," in \emph{Proc. AAAI Conf. Artif. Intell.}, 2018.

\bibitem{mmd-icml-2015} M. Long, Y. Cao, J. Wang, and M. I. Jordan, ``Learning transferable features with deep adaptation networks," in {\it  Proc. Int. Conf. Mach. Learn.}, pp. 97-105, 2015.

\bibitem{mtl-jmlr-2021} G. Blanchard, A. A. Deshmukh, Ü. Dogan, G. Lee, and C. Scott, ``Domain general- ization by marginal transfer learning," \emph{J. Mach. Learn. Res.}, vol. 22, pp. 1-55, Jan. 2021.

\bibitem{arm-nips-2021} M. Zhang, H. Marklund, N. Dhawan, A. Gupta, S. Levine, and C. Finn, ``Adaptive risk minimization: learning to adapt to domain shift," in {\it Proc. Advances Neural Inf. Process. Syst.}, pp. 23664-23678, 2021.

\bibitem{causalityinspired-cvpr-2022} F. Lv, J. Liang, S. Li, B. Zang, C. H. Liu, Z. Wang, and D. Liu2, ``Causality inspired representation learning for domain generalization," in {\it Proc. IEEE Conf. Comput. Vis. Pattern Recognit.}, pp. 8046-8056, 2022.

\bibitem{erm-tnn-1999} V. N. Vapnik, ``Domain general- ization by marginal transfer learning," \emph{IEEE Trans. Neural Netw.}, vol. 10, no. 5, pp. 988-999, Sept. 1999.

\bibitem{t-sne} L. v. d. Maaten and G. Hinton, ``Visualizing data using t-sne," \emph{J. Mach. Learn. Res.}, vol. 9, no. 11, pp. 2579-2605, Nov. 2008
	
	
\end{thebibliography}
\end{document}